\documentclass{article} 
\usepackage{iclr2021_conference,times}

\usepackage{amsmath,amsthm,amssymb,amsfonts}
\usepackage{siunitx}
\usepackage{bm}

\def\eqref#1{equation~\ref{#1}}

\def\1{\bm{1}}

\def\ra{{\textnormal{a}}}

\def\mQ{{\bm{Q}}}
\def\mR{{\bm{R}}}

\DeclareMathAlphabet{\mathsfit}{\encodingdefault}{\sfdefault}{m}{sl}
\SetMathAlphabet{\mathsfit}{bold}{\encodingdefault}{\sfdefault}{bx}{n}

\def\gC{{\mathcal{C}}}

\def\gN{{\mathcal{N}}}
\def\gO{{\mathcal{O}}}

\def\gW{{\mathcal{W}}}
\def\gX{{\mathcal{X}}}

\def\sP{{\mathbb{P}}}

\newcommand{\E}{\mathbb{E}}

\newcommand{\R}{\mathbb{R}}

\DeclareMathOperator*{\argmin}{arg\,min}

\DeclareMathOperator{\sign}{sign}

\DeclareMathOperator{\tr}{tr}
\DeclareMathOperator{\adj}{adj}
\DeclareMathOperator{\vol}{vol}
\DeclareMathOperator{\erf}{erf}

\newcommand{\ddt}{\frac{\partial}{\partial \theta}}
\newcommand{\dJdt}{\frac{\partial H}{\partial \theta}}
\usepackage{booktabs}
\usepackage{enumitem}
\usepackage{algorithm}
\usepackage[noend]{algpseudocode}

\newtheorem{definition}{Definition}
\newtheorem{proposition}{Proposition}

\usepackage{url}

\usepackage{thm-restate}

\usepackage[utf8]{inputenc}
\usepackage{pgfplots}
\DeclareUnicodeCharacter{2212}{−}
\usepgfplotslibrary{groupplots,dateplot}
\usetikzlibrary{patterns,shapes.arrows}
\pgfplotsset{compat=newest}
\usepackage{wrapfig}

\usepackage{cite} 
\usepackage[final]{hyperref} 
\hypersetup{
	colorlinks=true,       
	linkcolor=blue,        
	citecolor=blue,        
	filecolor=magenta,     
	urlcolor=blue         
}

\usepackage{hyperref}
\usepackage{url}
\usepackage{xspace}
\usepackage{subcaption}
\usepackage{amssymb}
\usepackage{cleveref}
\usepackage{booktabs}
\usepackage[font=small]{caption}
\usepackage{xcolor}

\newcommand*{\eg}{{\it e.g.}\@\xspace}
\newcommand*{\ie}{{\it i.e.}\@\xspace}

\renewcommand{\ra}[1]{\renewcommand{\arraystretch}{#1}}

\newcommand{\na}{\textcolor{gray}{\footnotesize N/A}}

\title{Convex Potential Flows: \\Universal Probability Distributions with \\Optimal Transport and Convex Optimization}

\author{Chin-Wei Huang \\
University of Montreal \& Mila \\
\texttt{chin-wei.huang@umontreal.ca}
\And
Ricky T. Q. Chen \\
University of Toronto \& Vector Institute \hspace{0.67cm} \\
\texttt{rtqichen@cs.toronto.edu}
\And
Christos Tsirigotis \\
University of Montreal \& Mila \\
\texttt{christos.tsirigotis@umontreal.ca}
\And
Aaron Courville \\
University of Montreal, Mila \& CIFAR Fellow \\
\texttt{aaron.courville@umontreal.ca}
}

\iclrfinalcopy 
\begin{document}

\maketitle

\begin{abstract}
Flow-based models are powerful tools for designing probabilistic models with tractable density. 
This paper introduces Convex Potential Flows (CP-Flow), a natural and efficient parameterization of invertible models inspired by the optimal transport (OT) theory.
CP-Flows are the gradient map of a strongly convex neural potential function. 
The convexity implies invertibility and allows us to resort to convex optimization to solve the convex conjugate for efficient inversion. 
To enable maximum likelihood training, we derive a new gradient estimator of the log-determinant of the Jacobian, which involves solving an inverse-Hessian vector product using the conjugate gradient method. 
The gradient estimator has \emph{constant-memory} cost, and can be made effectively \emph{unbiased} by reducing the error tolerance level of the convex optimization routine. 
Theoretically, we prove that CP-Flows are \emph{universal} density approximators and are \emph{optimal} in the OT sense. 
Our empirical results show that CP-Flow performs competitively on standard benchmarks of density estimation and variational inference.
\end{abstract}

\section{Introduction}

Normalizing flows \citep{dinh2014nice, rezende2015variational} have recently gathered much interest within the machine learning community, ever since its recent breakthrough in modelling high dimensional image data \citep{dinh2017density, kingma2018glow}. 
They are characterized by an invertible mapping that can reshape the distribution of its input data into a simpler or more complex one. 
To enable efficient training, numerous tricks have been proposed to impose structural constraints on its parameterization, such that the density of the model can be tractably computed. 

We ask the following question: ``what is the natural way to parameterize a normalizing flow?''
To gain a bit more intuition, we start from the one-dimension case. 
If a function $f:\R\rightarrow\R$ is continuous, it is invertible (injective onto its image) if and only if it is strictly monotonic. 
This means that if we are only allowed to move the probability mass continuously without flipping the order of the particles, then we can only rearrange them by changing the distance in between.

In this work, we seek to generalize the above intuition of monotone rearrangement in 1D.
We do so by motivating the parameterization of normalizing flows from an optimal transport perspective, which allows us to define some notion of rearrangement cost \citep{villani2008optimal}. 
It turns out, if we want the output of a flow to follow some desired distribution, under mild regularity conditions, we can characterize the unique optimal mapping by a convex potential \citep{brenier1991polar}. 
In light of this, we propose to parameterize normalizing flows by the gradient map of a (strongly) convex potential.
Owing to this theoretical insight, the proposed method is provably \emph{universal} and \emph{optimal}; this means the proposed flow family can approximate arbitrary distributions and requires the least amount of transport cost.
Furthermore, the parameterization with convex potentials allows us to formulate model inversion and gradient estimation as convex optimization problems. 
As such, we make use of existing tools from the convex optimization literature to cheaply and efficiently estimate all quantities of interest.

In terms of the benefits of parameterizing a flow as a gradient field, the convex potential is an $\R^d\rightarrow\R$ function, which is different from most existing discrete-time flows which are $\R^d\rightarrow\R^d$. 
This makes CP-Flow relatively compact. 
It is also arguably easier to design a convex architecture, as we do not need to satisfy constraints such as orthogonality or Lipschitzness; the latter two usually require a direct or an iterative reparameterization of the parameters.
Finally, it is possible to incorporate additional structure such as equivariance \citep{cohen2016group, zaheer2017deep} into the flow's parameterization, making CP-Flow a more flexible general purpose density model. 

% ...



\section{Background: Normalizing Flows and Optimal Transport}

Normalizing flows are characterized by a differentiable, invertible neural network $f$ such that the probability density of the network's output can be computed conveniently using the change-of-variable formula
\begin{align}
p_Y(f(x)) = p_X(x)\left|\frac{\partial f(x)}{\partial x}\right|^{-1}  
\qquad\Longleftrightarrow\qquad 
p_Y(y) = p_X(f^{-1}(y))\left|\frac{\partial f^{-1}(y)}{\partial y}\right| 
\label{eq:cov}
\end{align}
where the Jacobian determinant term captures the local expansion or contraction of the density near $x$ (resp. $y$) induced by the mapping $f$ (resp. $f^{-1}$), and $p_X$ is the density of a random variable $X$. 
The invertibility requirement has 
led to the design of many
special neural network parameterizations such as triangular maps, ordinary differential equations, orthogonality or Lipschitz constraints.


\paragraph{Universal Flows}
\label{sec:background}
For a general learning framework to be meaningful, a model needs to be flexible enough to capture variations in the data distribution.
In the context of density modeling, this corresponds to the model's capability to represent arbitrary probability distributions of interest. 
Even though there exists a long history of literature on universal approximation capability of deep neural networks \citep{cybenko1989approximation,lu2017expressive,lin2018resnet}, 
invertible neural networks generally have limited expressivity and cannot 
approximate
arbitrary functions.
However, for the purpose of approximating a probability distribution,
it suffices to show 
that
the distribution induced by a normalizing flow is universal. 

Among many ways to establish distributional universality of flow based methods (\eg{} \citealt{huang2018neural, huang2020solving, teshima2020coupling, kong2020expressive}), one particular approach is to approximate a \emph{deterministic coupling} between probability measures. 
Given a pair of probability densities
$p_X$ and $p_Y$, a deterministic coupling is a mapping $g$ such that $g(X)\sim p_Y$ if $X\sim p_X$.
We seek to find a coupling that is invertible, or at least can be approximated by invertible mappings. 

\paragraph{Optimal Transport} Let $c(x,y)$ be a cost function. 
The \emph{Monge problem} \citep{villani2008optimal}
pertains to finding the optimal transport map $g$ that realizes the minimal expected cost 
\begin{align}
\label{eq:monge}
J_c(p_X, p_Y) = \inf_{\widetilde{g}:\widetilde{g}(X)\sim p_Y} \E_{X\sim p_X}[ c(X, \widetilde{g}(X))]
\end{align}
When the second moments of $X$ and $Y$ are both finite, and $X$ is regular enough (e.g. having a density), then the special case of $c(x,y)=||x-y||^2$ has an interesting solution, a celebrated theorem due to \citet{brenier1987decomposition, brenier1991polar}:
\begin{restatable}[\textbf{Brenier's Theorem}, Theorem 1.22 of \citet{santambrogio2015optimal}]{thm}{diffcvx}
Let $\mu,\nu$ be probability measures with a finite second moment, and assume $\mu$ has a Lebesgue density $p_X$. 
Then there exists a convex potential $G$ such that the gradient map $g:=\nabla G$ (defined up to a null set) uniquely solves the Monge problem in \cref{eq:monge} with the quadratic cost function $c(x,y) = ||x-y||^2$. 
\end{restatable}

Some recent works are also inspired by Brenier's theorem and utilize a convex potential to parameterize a critic model, starting from \citet{taghvaei20192}, and further built upon by \citet{makkuva2019optimal} who parameterize a generator with a convex potential and concurrently by \citet{korotin2019wasserstein}.
Our work sets itself apart from these prior works in that it is entirely likelihood-based, minimizing the (empirical) KL divergence as opposed to an approximate optimal transport cost. 

\section{Convex Potential Flows}
\label{sec:cpflow}

\begin{figure}
    \centering
    \begin{subfigure}[b]{0.24\linewidth}
        \includegraphics[width=\linewidth]{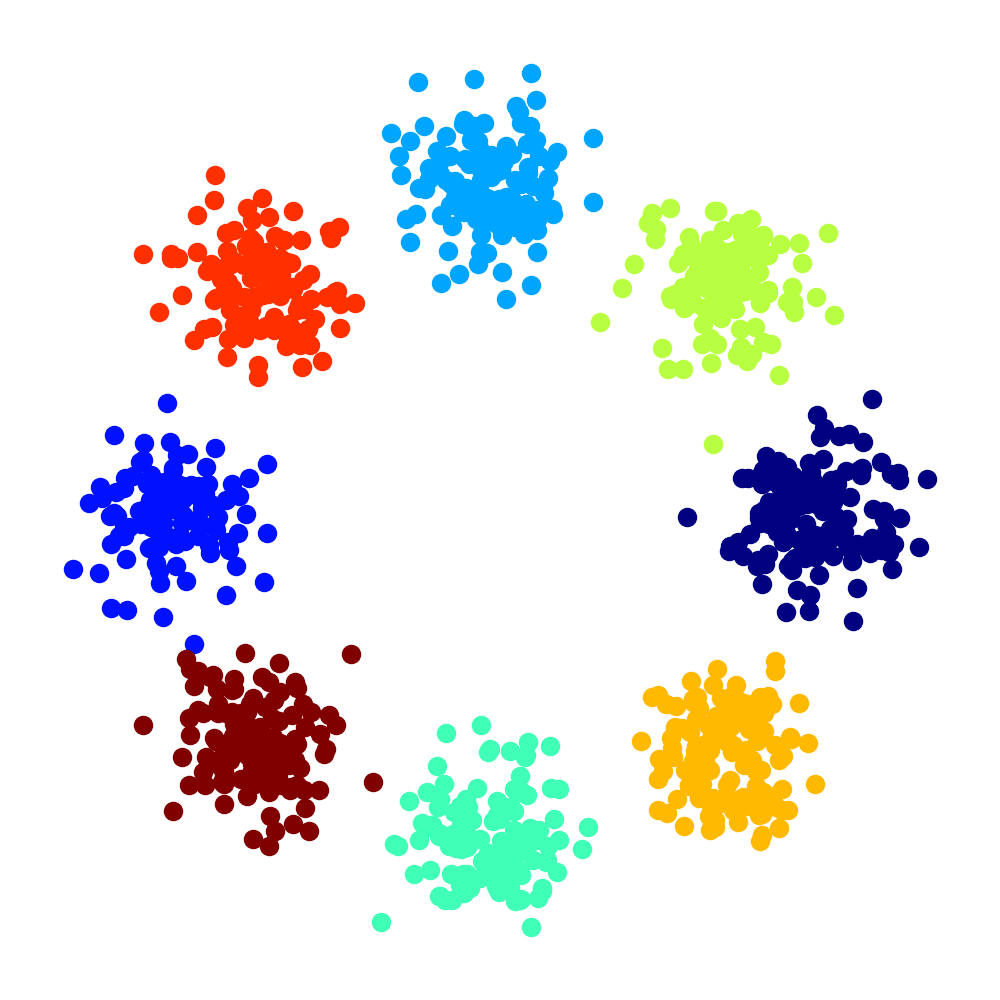}
        \caption{}
    \end{subfigure}%
    \begin{subfigure}[b]{0.24\linewidth}
        \includegraphics[width=\linewidth]{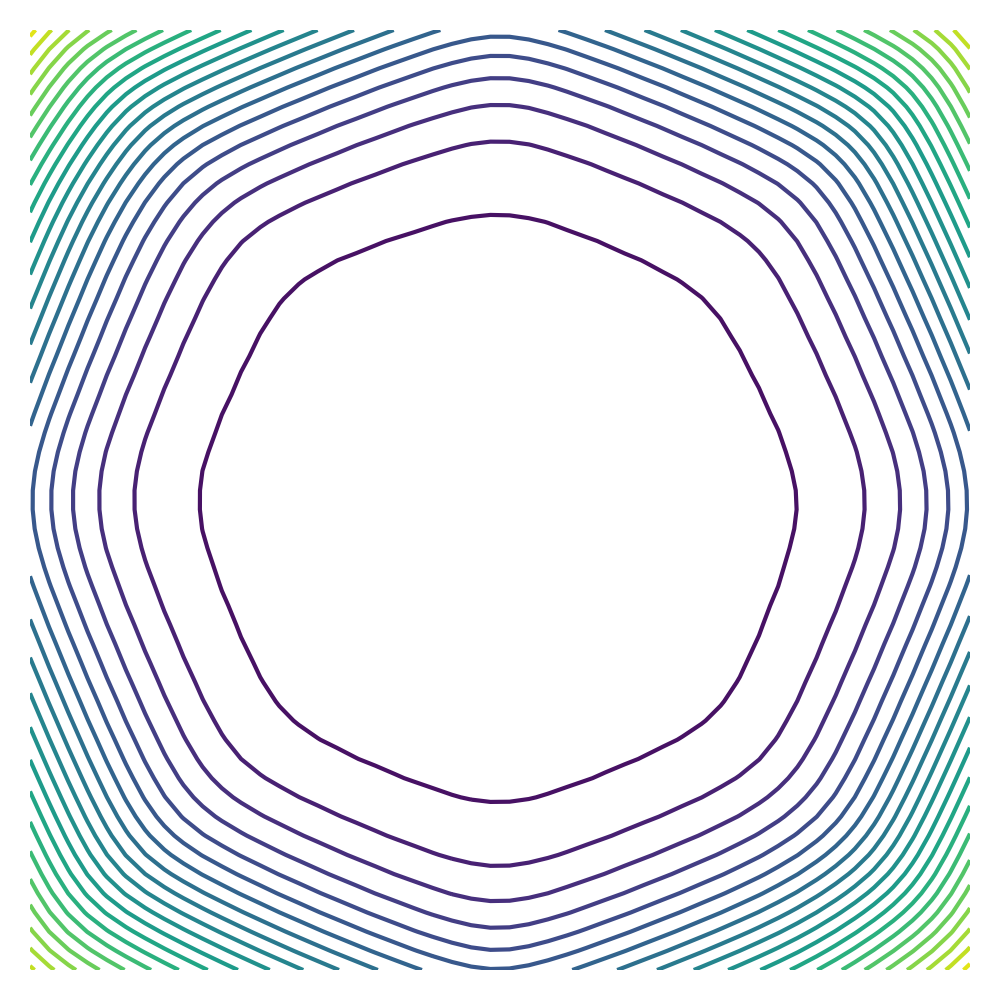}
        \caption{}
    \end{subfigure}%
    \begin{subfigure}[b]{0.24\linewidth}
        \includegraphics[width=\linewidth]{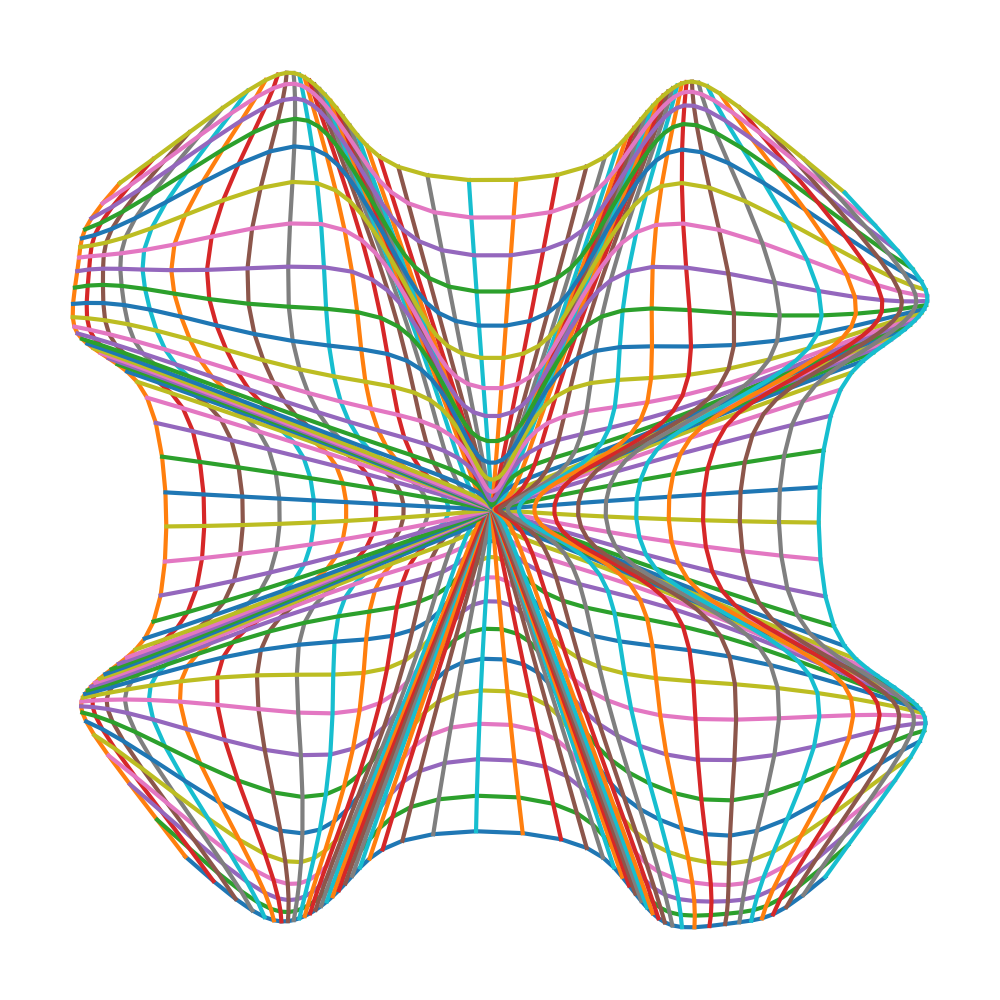}
        \caption{}
    \end{subfigure}%
    \begin{subfigure}[b]{0.24\linewidth}
        \includegraphics[width=\linewidth]{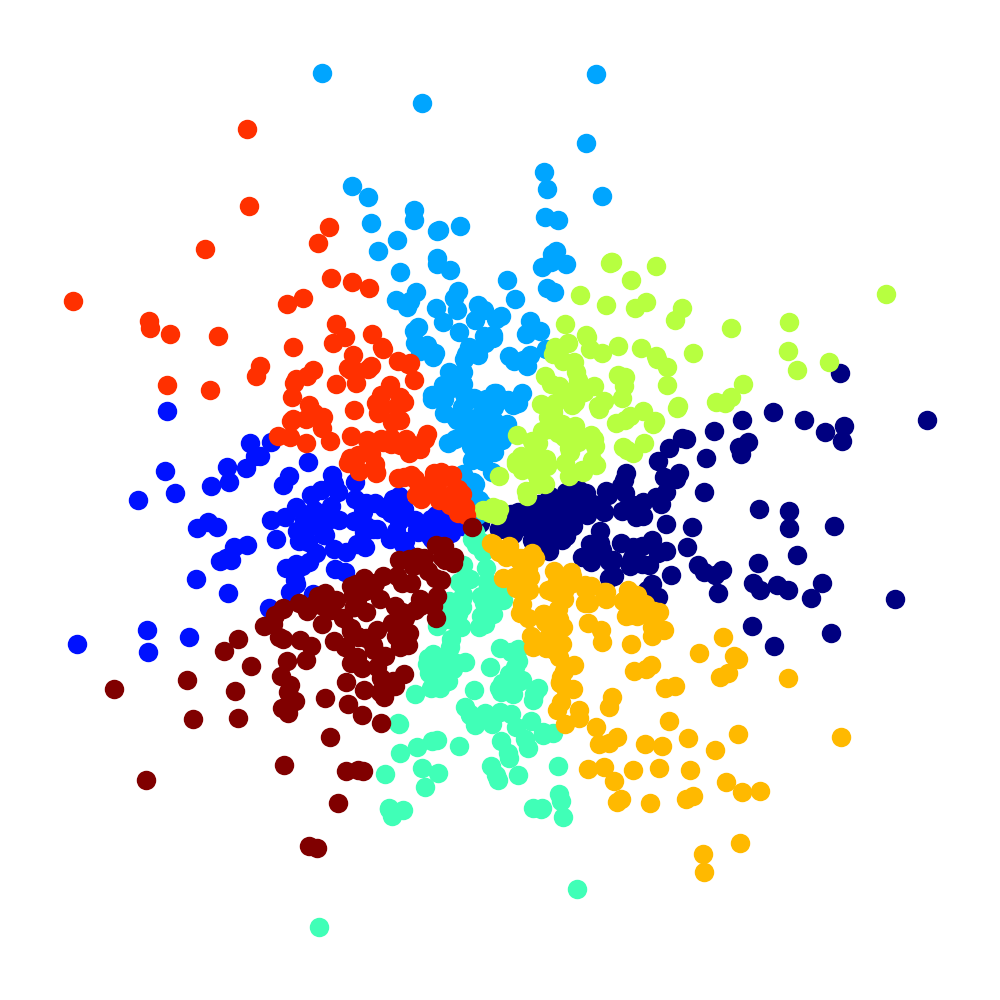}
        \caption{}
    \end{subfigure}
    \caption{\small Illustration of Convex Potential Flow. 
    (a) Data $x$ drawn from a mixture of Gaussians. 
    (b) Learned convex potential $F$. 
    (c) Mesh grid distorted by the gradient map of the convex potential $f=\nabla F$.
    (d) Encoding of the data via the gradient map $z=f(x)$.
    Notably, the encoding is the \emph{value of the gradient} of the convex potential.
    When the curvature of the potential function is locally flat, gradient values are small and this results in a contraction towards the origin.}
    \label{fig:8gaussian}
\end{figure}

Given a strictly convex potential $F$, we can define an injective map (invertible from its image) via its gradient $f=\nabla F$, since the Jacobian of $f$ is the Hessian matrix of $F$, and is thus positive definite.
In this section, we discuss the parameterization of the convex potential $F$ (\ref{sec:modeling}), 
and then address gradient estimation for CP-Flows (\ref{sec:grad_estimator}). 
We examine the connection to other parameterization of normalizing flows (\ref{sec:connection}), and finally rigorously prove universality in the next section. 


\subsection{Modeling}
\label{sec:modeling}

\paragraph{Input Convex Neural Networks} We use $L(x)$ to denote a linear layer,
and $L^+(x)$ to denote a linear layer with positive weights. 
We use the (fully) input-convex neural network (ICNN, \citet{amos2017input}) to parameterize the convex potential, which has the following form
$$F(x) = L_{K+1}^+(s(z_K)) + L_{K+1}(x) \qquad  z_k:= L_{k}^+ (s(z_{k-1})) + L_k(x) \qquad z_1 := L_1(x)$$
where $s$ is a non-decreasing, convex activation function. 
In this work, we use softplus-type activation functions, which is a rich family of activation functions that can be shown to uniformly approximate the ReLU activation. 
See Appendix~\ref{app:softplus} for details.

\begin{wrapfigure}[10]{r}{0.45\textwidth}
\vspace{-2.0em}
\begin{minipage}{0.45\textwidth}
\begin{algorithm}[H]
\caption{Inverting CP-Flow.}
\label{alg:invert}
\begin{algorithmic}[1]
\Procedure{Invert}{$F, y, \texttt{CvxSolver}$}       
    \State Initialize $x\leftarrow y$
    \State \textbf{def} \texttt{closure}():
        \State \hskip1.5em Compute loss: $l \leftarrow F(x) - y^\top x$
        \State \hskip1.5em \textbf{return} $l$
    \State $x\leftarrow\texttt{CvxSolver}(\texttt{closure}, x)$
    \State \textbf{return} $x$
\EndProcedure
\end{algorithmic}
\end{algorithm}
\end{minipage}
\end{wrapfigure}
\paragraph{Invertibility and Inversion Procedure}
If the activation $s$ is twice differentiable, then the Hessian $H_F$ is positive semi-definite.
We can make it strongly convex by adding a quadratic term
$F_\alpha(x) = \frac{\alpha}{2} ||x||^2_2 + F(x)$, such that $H_{F_\alpha}\succeq\alpha I\succ0$.
This means the gradient map $f_\alpha=\nabla F_\alpha$ is injective onto its image.
Furthermore, it is surjective since for any $y\in\R^d$, the potential $x\mapsto F_\alpha(x) - y^\top x$ has a unique minimizer\footnote{The minimizer $x^*$ corresponds to the gradient map of the \emph{convex conjugate} of the potential. 
See Appendix \ref{app:invertibility} for a formal discussion.} satisfying the first order condition $\nabla F_\alpha(x) = y$, due to the strong convexity 
and differentiability.
We refer to this invertible mapping $f_\alpha$ as the \emph{convex potential flow}, or the CP-Flow. 
The above discussion also implies we can plug in a black-box convex solver to invert the gradient map $f_\alpha$, which we summarize in Algorithm \ref{alg:invert}. 
Inverting a batch of independent inputs is as simple as summing the convex potential over all inputs: since all of the entries of the scalar $l$ in the minibatch are independent of each other, computing the gradient all $l$'s wrt all $x$'s amounts to computing the gradient of the summation of $l$'s wrt all $x$'s. 
Due to the convex nature of the problem, a wide selection of algorithms can be used with convergence guarantees \citep{nesterov1998introductory}. 
In practice, we use the \emph{L-BFGS} algorithm \citep{byrd1995limited} as our \texttt{CvxSolver}.

\newpage 

\paragraph{Estimating Log Probability} Following equation (\ref{eq:cov}), computing the log density for CP-Flows requires taking 
the log determinant of a symmetric positive definite Jacobian matrix (as it is the Hessian of the potential). 
There exists numerous works on estimating spectral densities~(\eg{} \citealp{tal1984accurate,silver1994densities,han2018stochastic,adams2018estimating}), of which this quantity is a special case. See \citet{lin2016approximating} for an overview of methods that only require access to Hessian-vector products.
Hessian-vector products (hvp) are cheap to compute with reverse-mode automatic differentiation~\citep{baydin2017automatic}, which does not require constructing the full Hessian matrix and has the same asymptotic cost as evaluating $F_\alpha$.

In particular, the log determinant can be rewritten in the form of a generalized trace $\tr \log H$. \citet{chen2019residual} limit the spectral norm (\ie eigenvalues) of $H$ and directly use the Taylor expansion of the matrix logarithm. Since our $H$ has unbounded eigenvalues, we use a more complex algorithm designed for symmetric matrices, the \emph{stochastic Lanczos quadrature} (SLQ; \citealp{ubaru2017fast}). At the core of SLQ is the Lanczos method, which computes $m$ eigenvalues of $H$ by first constructing a symmetric tridiagonal matrix $T \in \R^{m \times m}$ and computing the eigenvalues of $T$. The Lanczos procedure only requires Hessian-vector products, and it can be combined with a stochastic trace estimator to provide a stochastic estimate of our log probability. We chose SLQ because it has shown theoretically and empirically to have low variance~\citep{ubaru2017fast}. 

\subsection{\texorpdfstring{$\gO(1)$}--Memory Unbiased \texorpdfstring{$\nabla\log\det H$}{logdet gradient} estimator}
\label{sec:grad_estimator}

\begin{wrapfigure}[10]{r}{0.4\textwidth}
\vspace{-2em}
\includegraphics[width=\linewidth]{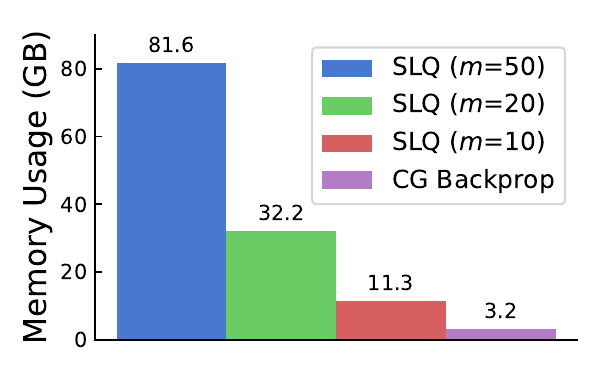}
\vspace{-2em}
\caption{\small Memory for training CIFAR-10.}
\label{fig:memory_profile}
\end{wrapfigure}
We would also like to have an estimator for the \emph{gradient} of the log determinant to enable variants of stochastic gradient descent for optimization.
Unfortunately, directly backpropagating through the log determinant estimator is not ideal.
Two major drawbacks of directly differentiating through SLQ are that it requires (i) differentiating through an eigendecomposition routine and (ii) storing all Hessian-vector products in memory~(see \cref{fig:memory_profile}).
Problem (i) is more specific to SLQ, 
because
the gradient of an eigendecomposition is not defined when the eigenvalues are not unique~\citep{seeger2017auto}. Consequently, we have empirically observed that differentiating through SLQ can be unstable, frequently resulting in NaNs due to the eigendecomposition.
Problem (ii) will hold true for other algorithms that also estimate $\log \det H$ with Hessian-vector products, and generally the only difference is that a different numerical routine would need to be differentiated through. 
Due to these problems, we do not differentiate through SLQ, but we still use it as an efficient method for monitoring training progress.

Instead, it is possible to construct an alternative formulation of the gradient as the solution of a convex optimization problem, foregoing the necessity of differentiating through an estimation routine of the log determinant.
We adapt the gradient formula from~\citet[Appendix~C]{chen2019residual} to the context of convex potentials.
Using Jacobi's formula$^*$ and the adjugate representation of the matrix inverse$^\dagger$, for any invertible matrix $H$ with parameter $\theta$, we have the following identity:
\begin{align}
\resizebox{0.99\textwidth}{!}{%
$\ddt \log\det H 
= \frac{1}{\det H} \ddt\det H 
\overset{*}= \frac{1}{\textcolor{red}{\det H}} \tr\left(\textcolor{red}{\adj(H)} \dJdt\right) 
\overset{\dagger}= \tr\left( \textcolor{red}{H^{-1}} \dJdt\right) 
= \E_v\left[v^\top H^{-1} \dJdt v\right].
$}
\label{eq:ldj_grad_estimator}
\end{align}
Notably, in the last equality, we used the Hutchinson trace estimator \citep{hutchinson1989stochastic} with a Rademacher random vector $v$, leading to a $\gO(1)$-memory, unbiased Monte Carlo gradient estimator.

Computing the quantity $v^\top H^{-1}$ in \cref{eq:ldj_grad_estimator} by constructing and inverting the full Hessian requires $d$ calls to an automatic differentiation routine and is too costly for our purposes.
However, 
we can recast this quantity as the solution of a quadratic optimization problem 
\begin{equation}\label{eq:quad_opt}
\arg\min_z \left\{ \frac{1}{2} z^\top H z - v^\top z\right\}
\end{equation}
which has the unique minimizer $z^* = H^{-1} v$ since $H$ is symmetric positive definite.

\newpage

\begin{wrapfigure}[10]{r}{0.48\textwidth}
\vspace{-1.25em}
\begin{minipage}{0.48\textwidth}
\begin{algorithm}[H]
\caption{Surrogate training objective.} 
\label{alg:logdet_grad}
\begin{algorithmic}[1]
\Procedure{SurrogateObj}{$F, x, \texttt{CG}$}       
    \State Obtain the gradient $\smash{f(x) \triangleq \nabla_x F(x)}$
    \State Sample Rademacher random vector $r$
    \State \textbf{def} \texttt{hvp}($v$):
        \State \hskip1.5em \textbf{return} $v^\top \frac{\partial}{\partial x} f(x)$
    \State $z \leftarrow \texttt{stop\_gradient}\left(\texttt{CG}(\texttt{hvp}, r)\right)$
    \State \textbf{return} $\texttt{hvp}(z)^\top r$
\EndProcedure
\end{algorithmic}
\end{algorithm}
\end{minipage}
\end{wrapfigure}
We use the \emph{conjugate gradient} (CG) method, which is specifically designed for solving the unconstrained optimization problems in \cref{eq:quad_opt} with symmetric positive definite $H$.
It uses only Hessian-vector products and is straightforward to parallelize.
%
Conjugate gradient is guaranteed to return the exact solution $z^*$ within $d$ iterations, and the error of the approximation is known to converge exponentially fast $||z^m - z^*||_{H} \leq 2 \gamma^{m} ||z^0 - z^*||_{H} $, where $z^m$ is the estimate after $m$ iterations. The rate of convergence $\gamma<1$ relates to the condition number of $H$. 
For more details, see \citet[Ch.~5]{nocedal2006numerical}. In practice, we terminate CG when $||Hz^m - v||_\infty < \tau$ is satisfied for some user-controlled tolerance.
Empirically, we find that stringent tolerance values are unnecessary for stochastic optimization~(see \cref{app:ablation}).

Estimating the full quantity in \cref{eq:ldj_grad_estimator} is then simply a matter of computing and differentiating a scalar quantity (a surrogate objective) involving another Hessian-vector product: $\frac{d}{d\theta} \left((z^m)^\top H v\right)$, where only $H$ is differentiated through (since $z^m$ is only used to approximate $v^\top H^{-1}$ as a modifier of the gradient).
We summarize this procedure in Algorithm~\ref{alg:logdet_grad}. 
Similar to inversion, the \texttt{hvp} can also be computed in batch by summing over the data index, since all entries are independent.

\subsection{Connection to other normalizing flows}
\label{sec:connection}
\paragraph{Residual Flow} For $\alpha=1$, the gradient map $f_1$ resembles the residual flow \citep{behrmann2019invertible,chen2019residual}.
They require the residual block---equivalent to our gradient map $f$---to be \emph{contractive} (with Lipschitz constant strictly smaller than 1) as a sufficient condition for invertibility. In contrast, we enforce invertibility by using strongly convex potentials, which guarantees that the inverse of our flow is globally unique. With this, we do not pay the extra compute cost for having to satisfy Lipschitz constraints using methods such as spectral normalization~\citep{miyato2018spectral}.
Our gradient estimator is also derived similarly to that of~\citet{chen2019residual}, though we have the benefit of using well-studied convex optimization algorithms for computing the gradients.

\paragraph{Sylvester Flow}
By restricting the architecture of our ICNN to one hidden layer, we can also recover a form similar to Sylvester Flows. For a 1-hidden layer ICNN ($K=1$) and $\alpha=1$, we have $F_{1} = \frac{1}{2}||x||^2_2 + L_2^+ (s(L_1 x)) + L_2 (x)$.
Setting the weights of $L_2$ to zero, we have
\begin{equation}
f_1(x) = \nabla_x F_1(x) = x +  W_1^\top \text{diag}({w_2^+})  s'(W_1 x + b_1).
\end{equation}
We notice the above form bears a close resemblance to the Sylvester normalizing flow \citep{van2018sylvester} (with $\mQ$, $\mR$ and $\widetilde{\mR}$ from \citet{van2018sylvester} being equal to $W_1^\top$, $\text{diag}(w_2^+)$ and $I$, respectively).
For the Sylvester flow to be invertible, they require that $\mR$ and $\widetilde{\mR}$ be triangular and $\mQ$ be orthogonal, which is a computationally costly procedure. This orthogonality constraint also implies that the number of hidden units cannot exceed $d$. 
This restriction to orthogonal matrices and one hidden layer are for applying Sylvester's determinant identity. 
In contrast, we do not require our weight matrices to be orthogonal, and we can use any hidden width and depth for the ICNN.

\paragraph{Sigmoidal Flow}
Let $s$ be the softplus activation function and $\sigma=s'$.
Then for the 1-dimensional case ($d=1$) and $\alpha=0$ (without the residual connection), we have
\begin{equation}
    \frac{\partial}{\partial x}F_0(x) = \sum_{j=1} w_{1, j} w_{2, j}^+ \sigma( w_{1, j} x + b_{1, j} )
= \sum_{j=1} |w_{1, j}| w_{2, j}^+ \sigma( |w_{1, j}| x + \sign(w_{1, j}) b_{1, j} ) + \text{const.}
\end{equation}
which is equivalent to the sigmoidal flow of~\citet{huang2018neural} up to rescaling (since the weighted sum is no longer a convex sum) and a constant shift, and is monotone due to the positive weights.
This correspondence is not surprising since a differentiable function is convex if and only if its derivative is monotonically non-decreasing.
It also means we can parameterize an increasing function as the derivative of a convex function, which opens up a new direction for parameterizing autoregressive normalizing flows~\citep{kingma2016improved, huang2018neural, muller2019neural, jaini2019sum, durkan2019neural, wehenkel2019unconstrained}.

\paragraph{Flows with Potential Parameterization}
Inspired by connections between optimal transport and continuous normalizing flows, some works~\citep{zhang2018monge, finlay2020learning,onken2020ot} have proposed to parameterize continuous-time transformations by taking the gradient of a scalar potential.
They do not strictly require the potential to be convex since it is guaranteed to be invertible in the infinitesimal setting of continuous normalizing flows \citep{chen2018neural}. 
There exist works~\citep{yang2019potential,finlay2020train,onken2020ot} that have applied the theory of optimal transport to regularize continuous-time flows to have low transport cost.
In contrast, we connect optimal transport with discrete-time normalizing flows, and CP-Flow is guaranteed by construction to converge pointwise to the optimal mapping between distributions without explicit regularization (see Section~\ref{sec:theory}). 



\section{Theoretical Analyses}
\label{sec:theory}
As explained in Section \ref{sec:background}, the parameterization of CP-Flow is inspired by the Brenier potential. 
So naturally we would hope to show that (1) CP-Flows are distributionally universal, and that (2) the learned invertible map is optimal in the sense of the average squared distance the input travels $\E[||x-f(x)||^2]$.
Proofs of statements made in this section can be found in Appendices \ref{app:univ} and \ref{app:optim}.

To show (1), our first step is to show that ICNNs can approximate arbitrary convex functions.
However, convergence of potential functions does not generally imply convergence of the gradient fields. 
A classic example is the sequence $F_n=\sin(nx)/\sqrt{n}$ and the corresponding derivatives $f_n=\cos(nx)\sqrt{n}$: $F_n\rightarrow 0$ as $n\rightarrow\infty$ but $f_n$ does not. 
Fortunately, convexity allows us to control the variation of the gradient map (since the derivative of a convex function is monotone), so our second step of approximation holds.

\vspace{0.2cm}
\begin{restatable}{thm}{diffcvx}
\label{thm:diff_convex_converge}
Let $F_n:\R^d\rightarrow\R$ be differentiable convex functions
and $G:\R^d\rightarrow\R$ be a proper convex function.
Assume $F_n\rightarrow G$.
Then for almost every $x\in\R^d$,
$G$ is differentiable and $f_n(x):=\nabla F_n(x) \rightarrow \nabla G(x) =: g(x)$. 
\end{restatable}

Combining these two steps and Brenier's theorem, we show that CP-Flow with softplus-type activation function is distributionally universal.

\vspace{0.2cm}
\begin{restatable}[\textbf{Universality}]{thm}{univ}
\label{thm:cpflow_universal}
Given random variables $X\sim\mu$ and $Y\sim\nu$, with $\mu$ being absolutely continuous w.r.t. the Lebesgue measure,
there exists a sequence of ICNN $F_n$ with a softplus-type activation, such that $\nabla F_n\circ X \rightarrow Y$ in distribution. 
\end{restatable}

\paragraph{N.B.} In the theorem we do not require the second moment to be finite, as for arbitrary random variables we can apply the standard truncation technique and redistribute the probability mass so that the new random variables are almost surely bounded.
For probability measures with finite second moments, we indeed use the gradient map of ICNN to approximate the optimal transport map corresponding to the Brenier potential. 
In the following theorem, we show that the optimal transport map is the only such mapping that we can approximate if we match the distributions.

\vspace{0.2cm}
\begin{restatable}[\textbf{Optimality}]{thm}{optim}
\label{thm:optimal}
Let $G$ be the Brenier potential of $X\sim\mu$ and $Y\sim\nu$, and let $F_n$ be a convergent sequence of differentiable, convex potentials, such that $\nabla F_n\circ X \rightarrow Y$ in distribution. 
Then $\nabla F_n$ converges almost surely to $\nabla G$.
\end{restatable}

The theorem states that 
in practice, even if we
optimize according to some loss that traces the convergence in distribution,
our model is still able to recover the optimal transport map, as if we were optimizing according to the transport cost. 
This allows us to estimate optimal transport maps without solving the constrained optimization in (\ref{eq:monge}). See \citet{seguy2018large} for some potential applications of the optimal transport map, such as domain adaptation or domain translation. 

\section{Experiment}
We use CP-Flow to perform density estimation (RHS of (\ref{eq:cov})) and variational inference (LHS of (\ref{eq:cov})) to assess its approximation capability, and the effectiveness of the proposed gradient estimator.
All the details of experiments can be found in Appendix \ref{app:exp}.
Code is available at \href{https://github.com/CW-Huang/CP-Flow}{https://github.com/CW-Huang/CP-Flow}.

\paragraph{ICNN Architecture}
Despite the universal property, having a poor parameterization can lead to difficulties in optimization and limit the effective expressivity of the model. 
We propose an architectural enhancement of ICNN, defined as follows (note the change in notation: instead of writing the pre-activations $z$, we use $h$ to denote the activated units):
\begin{align}
F^{aug}(x)&:=L_{K+1}^+ (h_K) + L_{K+1} (x) \nonumber
\\
h_k := \texttt{concat}([\widetilde{h}_k, h_k^{aug}])
\qquad
\widetilde{h}_k &:= s(L_{k}^+ (h_{k-1}) + L_{k} (x))
\qquad
h_k^{aug} = s(L_{k}^{aug} (x))
\end{align}
where half of the hidden units are directly connected to the input, so the gradient would have some form of skip connection. 
We call this the input-augmented ICNN. 
Unless otherwise stated, we use the input-augmented ICNN as the default architecture.

\subsection{Toy examples}
\begin{wrapfigure}[15]{r}{0.5\textwidth}
\vspace{-3.0em}
\begin{center}
    \centering
    \includegraphics[width=0.12\textwidth]{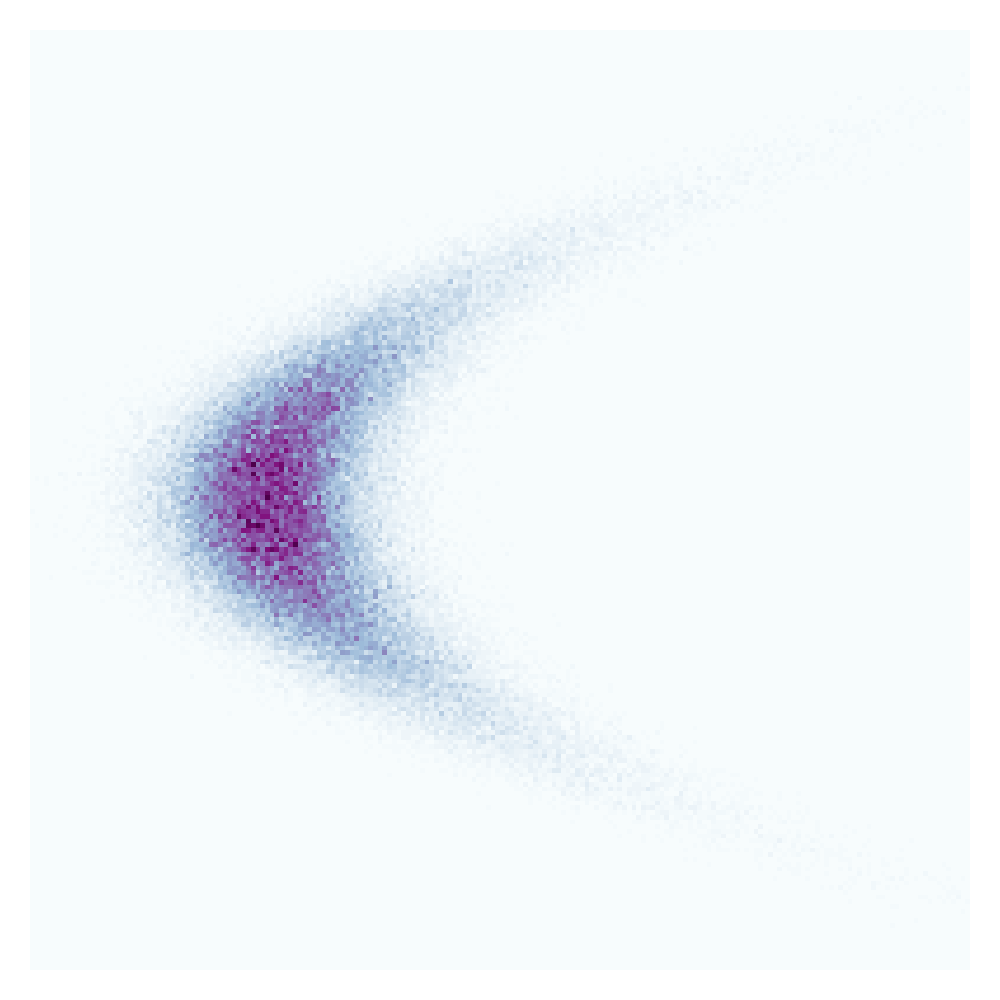}
    \includegraphics[width=0.12\textwidth]{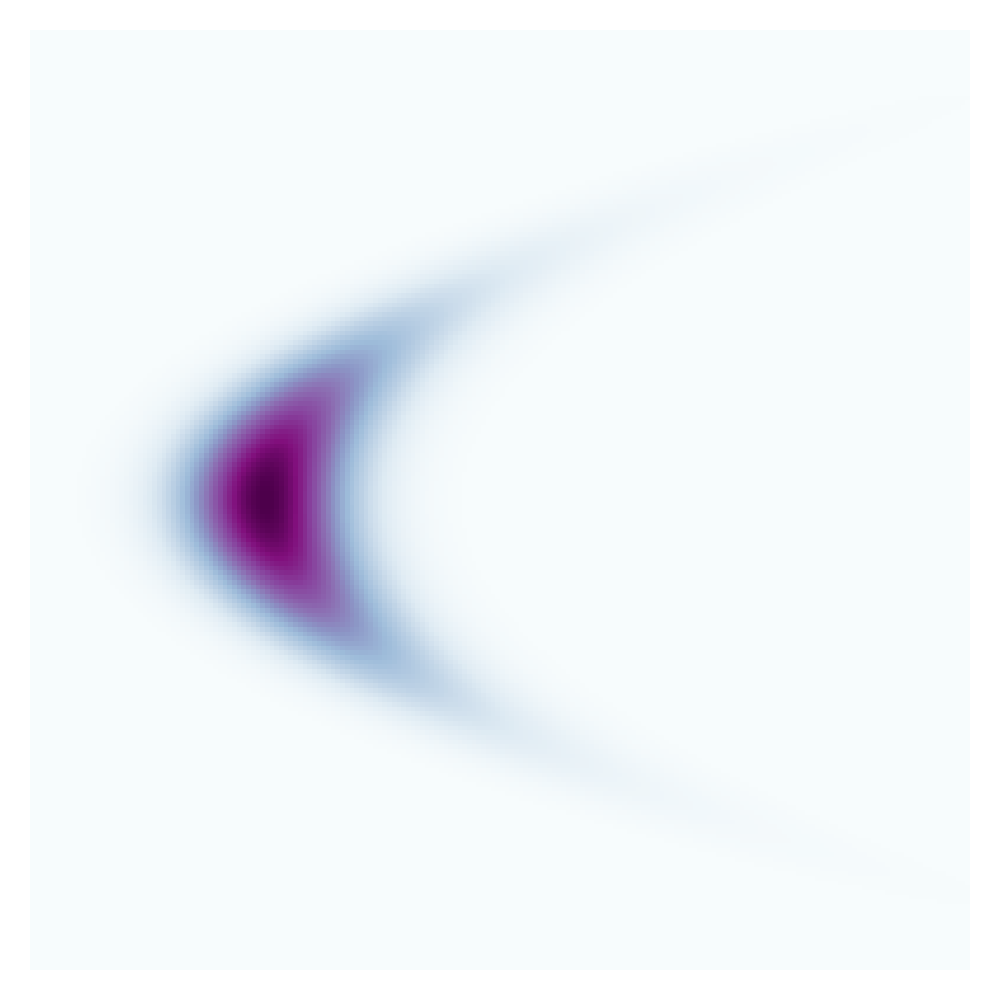}
    \includegraphics[width=0.12\textwidth]{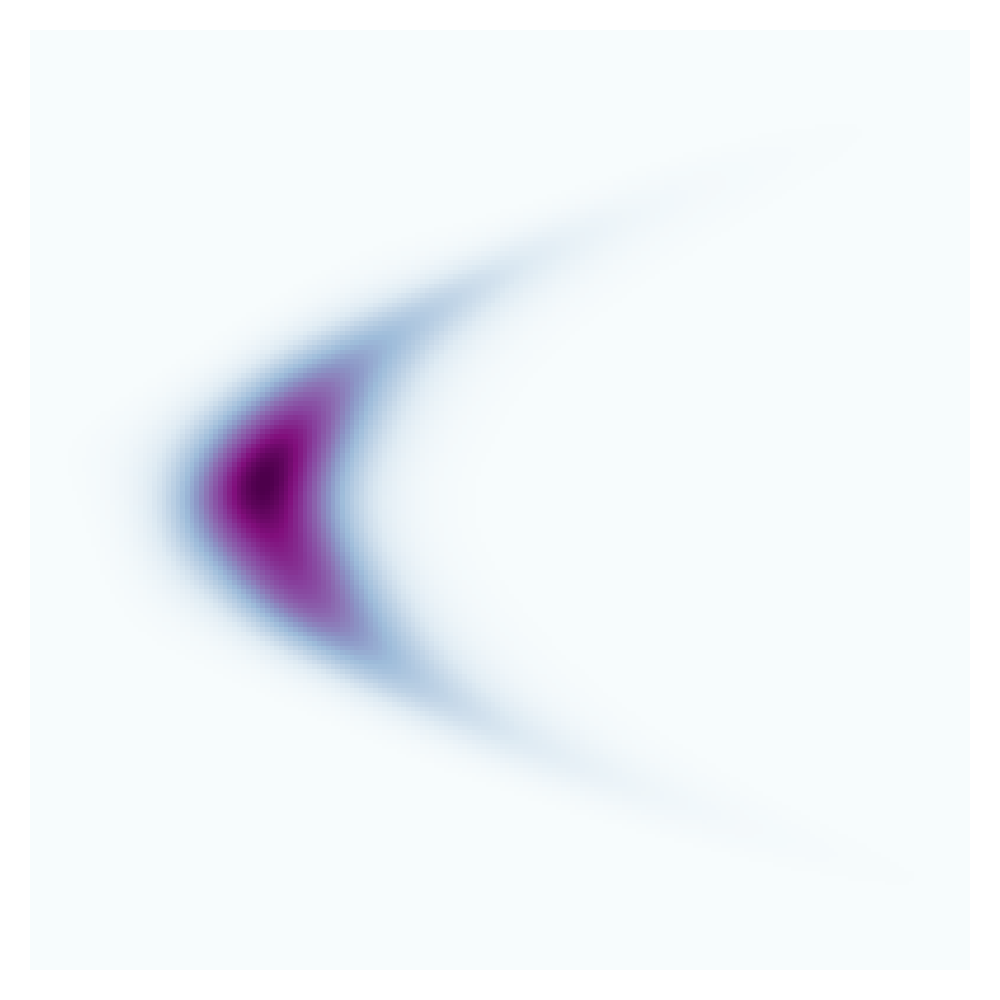}
    \includegraphics[width=0.12\textwidth]{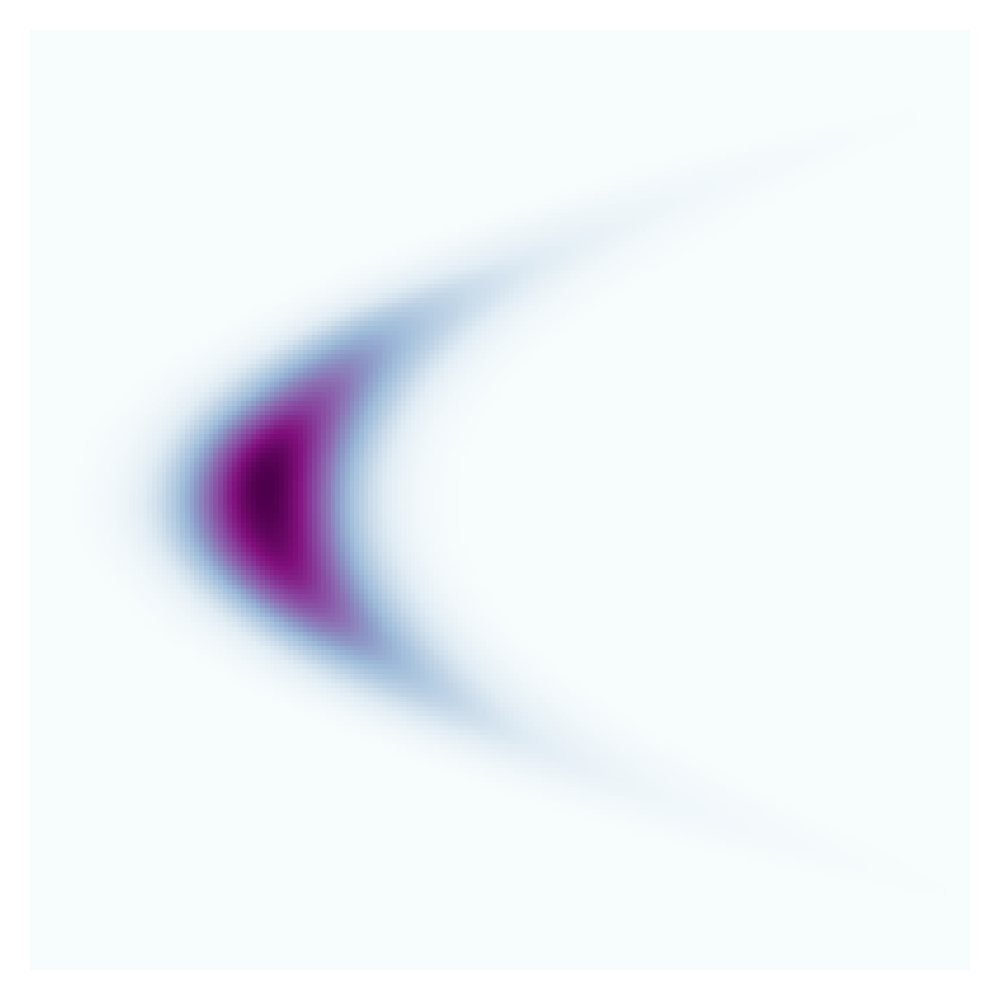} \\
    \includegraphics[width=0.12\textwidth]{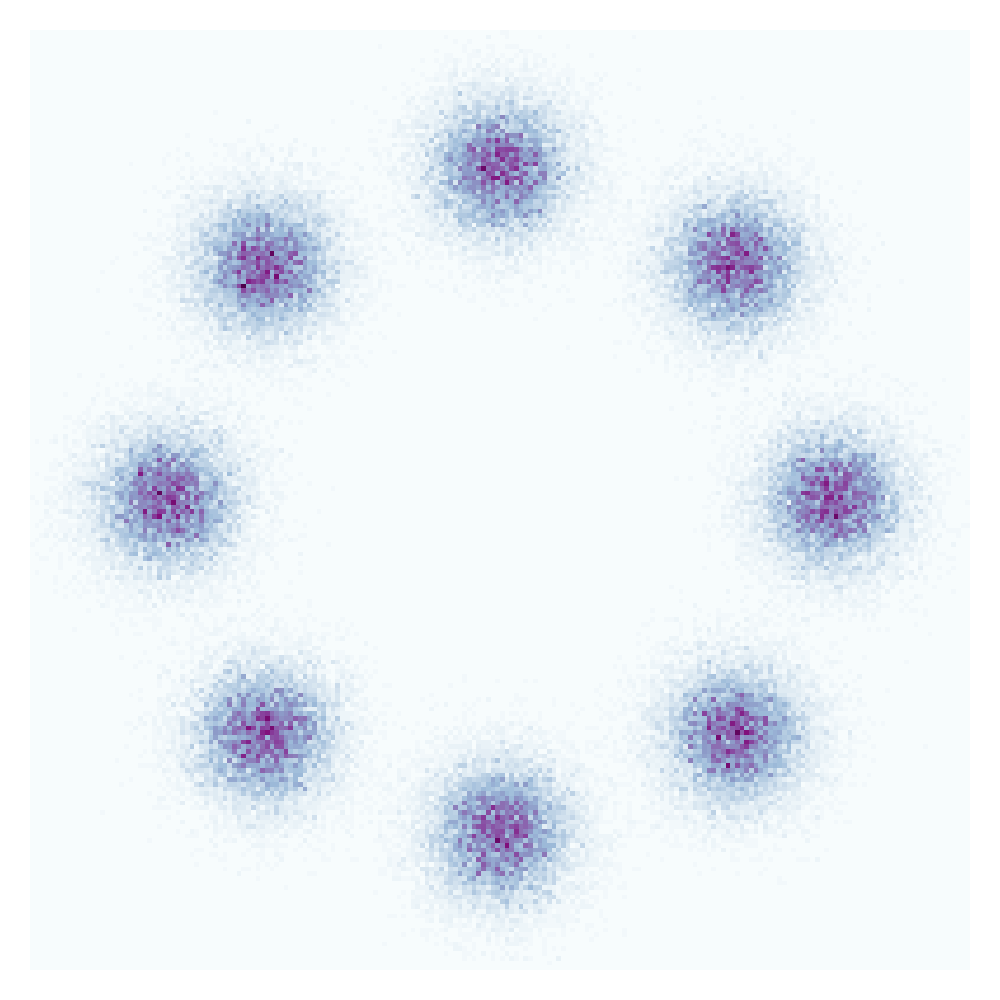}
    \includegraphics[width=0.12\textwidth]{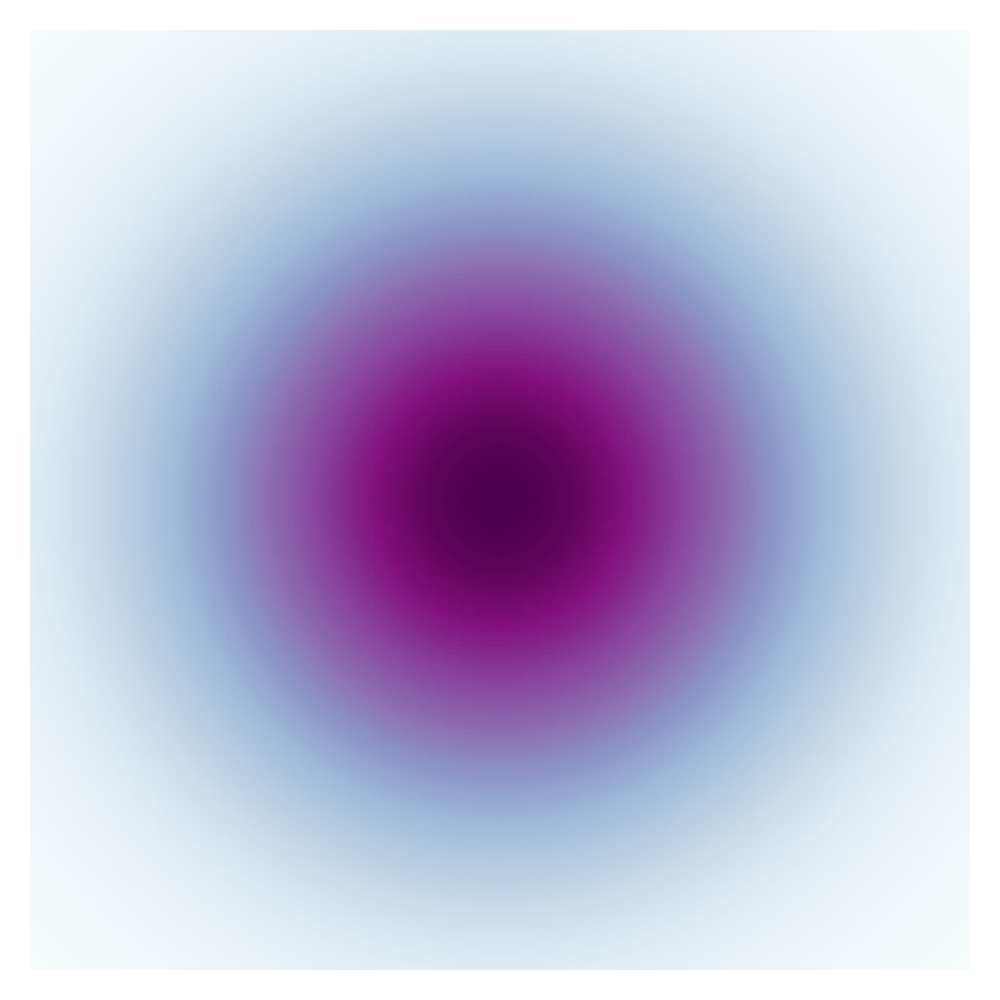}
    \includegraphics[width=0.12\textwidth]{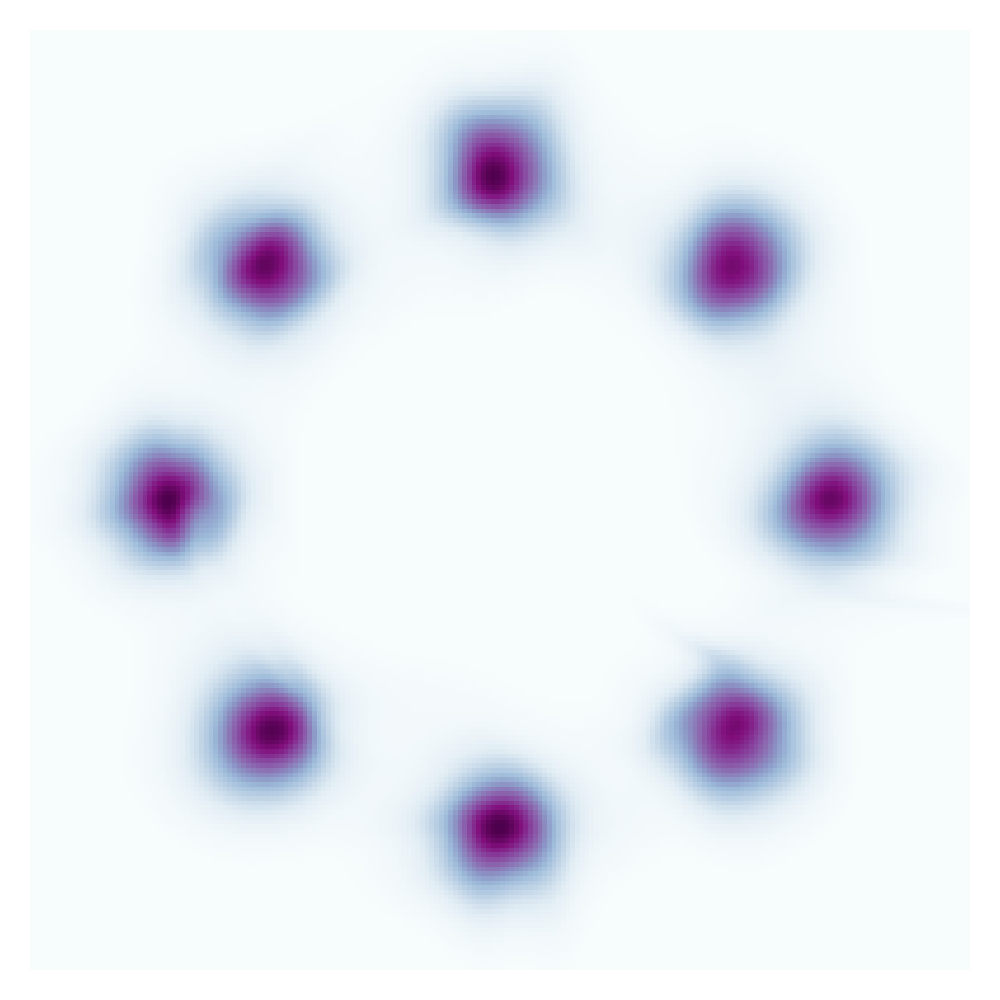}
    \includegraphics[width=0.12\textwidth]{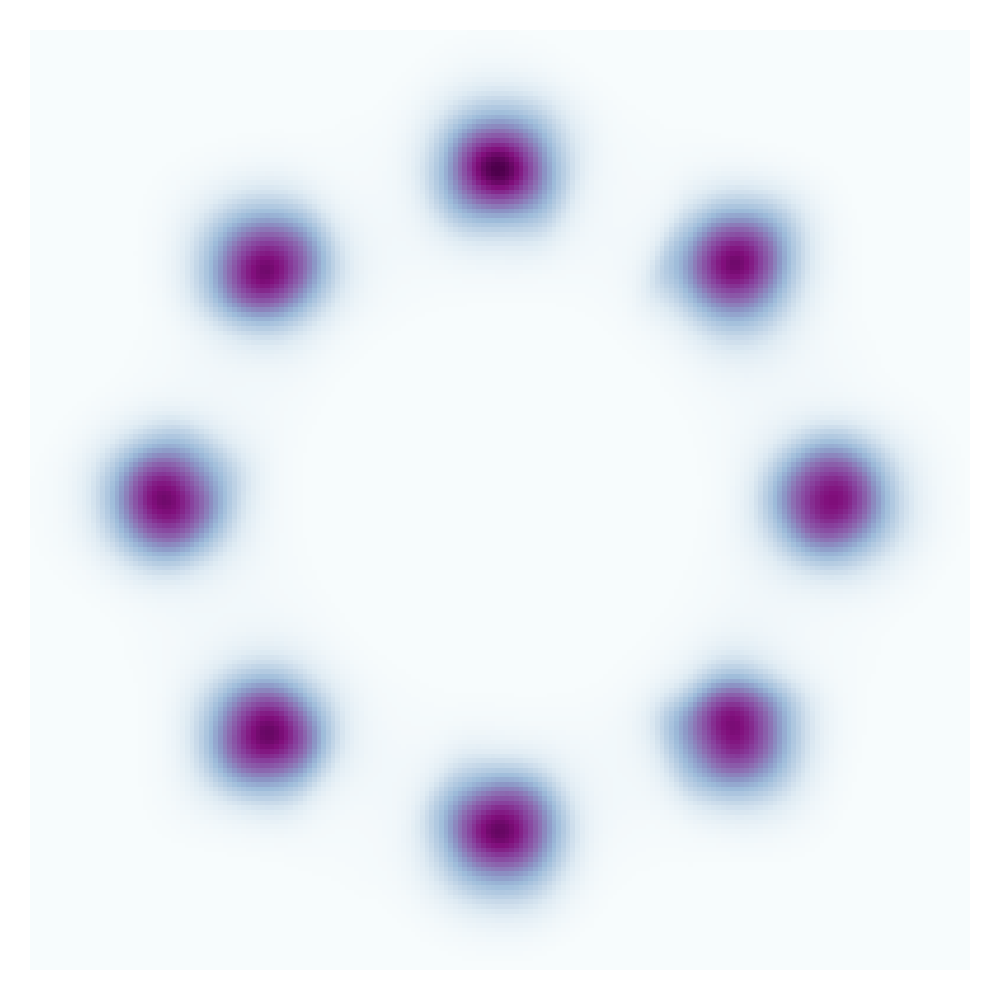} \\
    \includegraphics[width=0.12\textwidth]{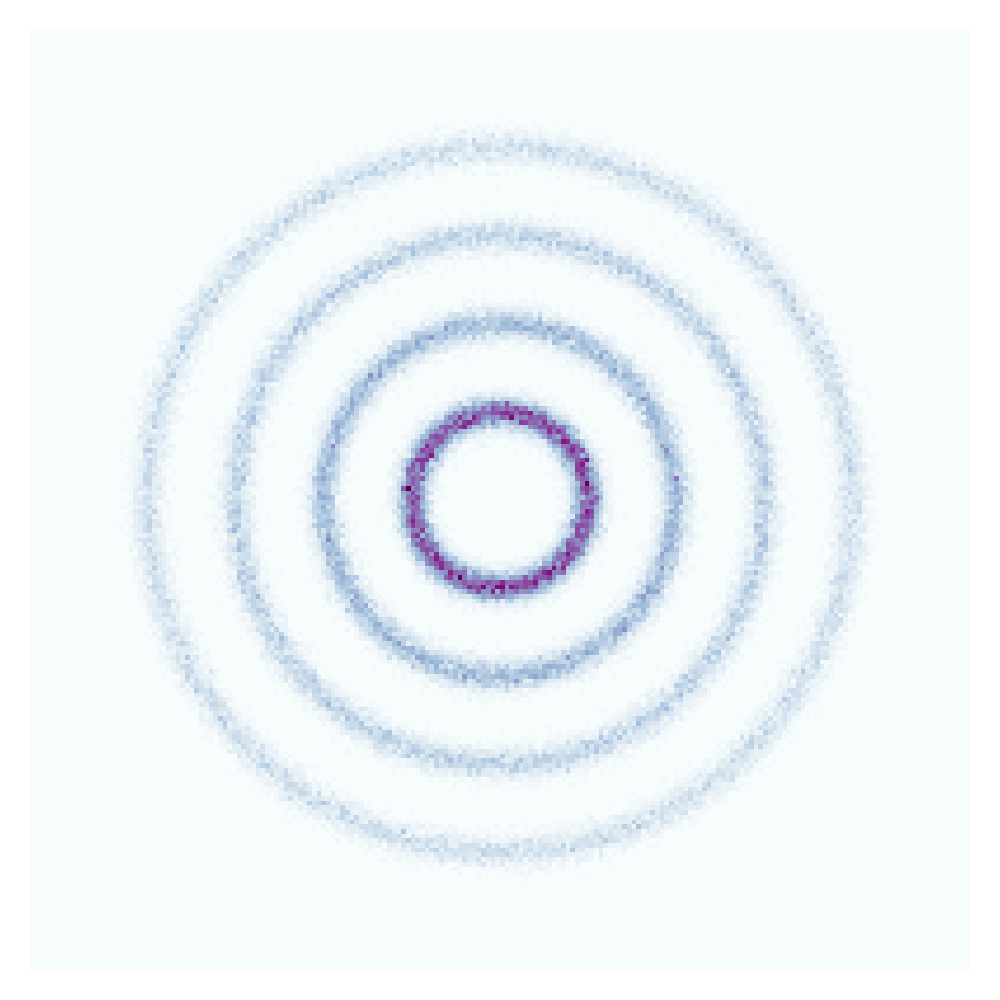}
    \put(-32,-5){\small Data}
    \includegraphics[width=0.12\textwidth]{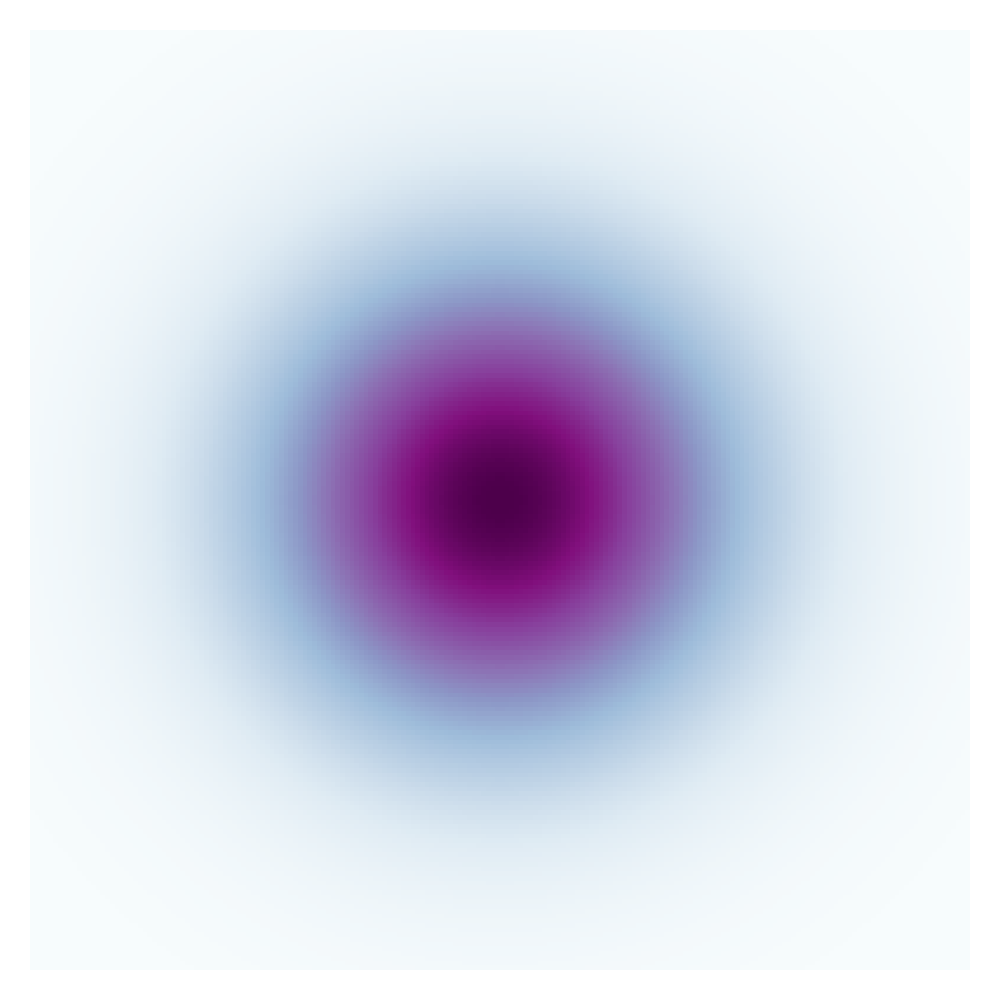}
    \put(-32,-5){\small MAF}
    \includegraphics[width=0.12\textwidth]{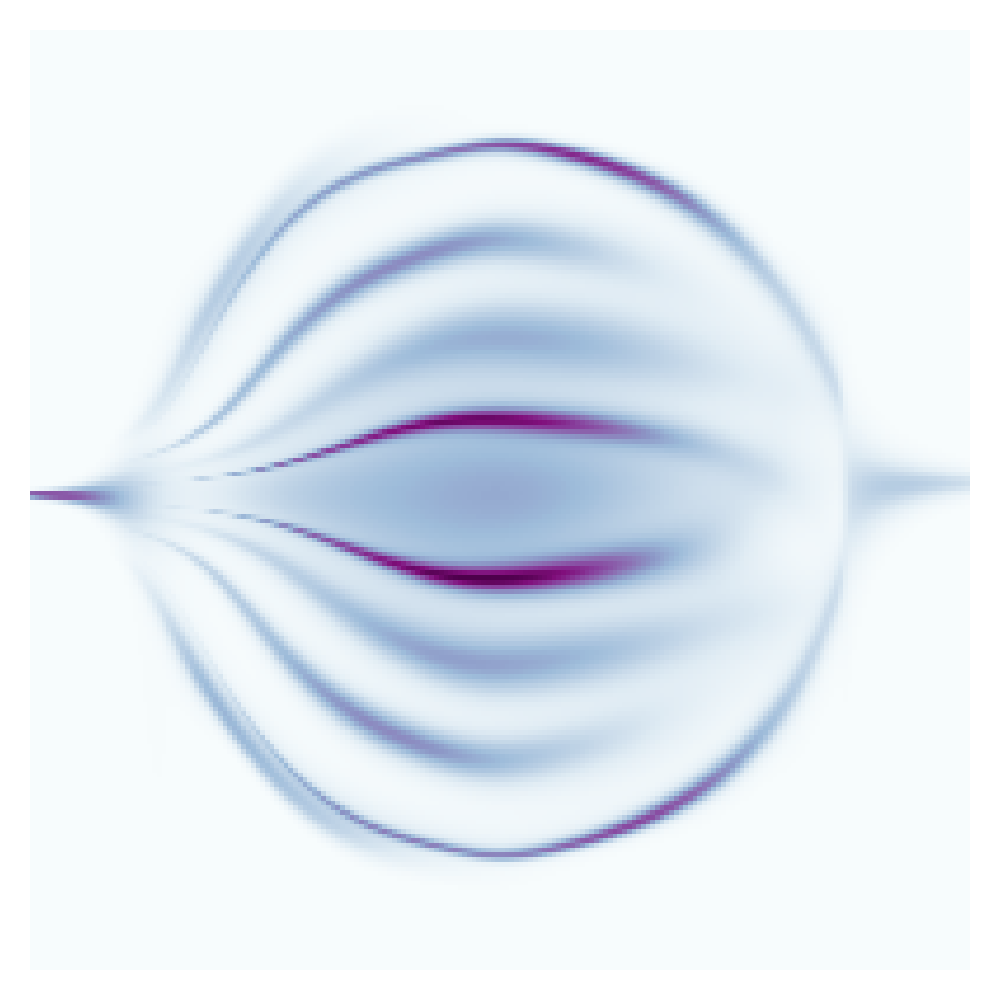}
    \put(-32,-5){\small NAF}
    \includegraphics[width=0.12\textwidth]{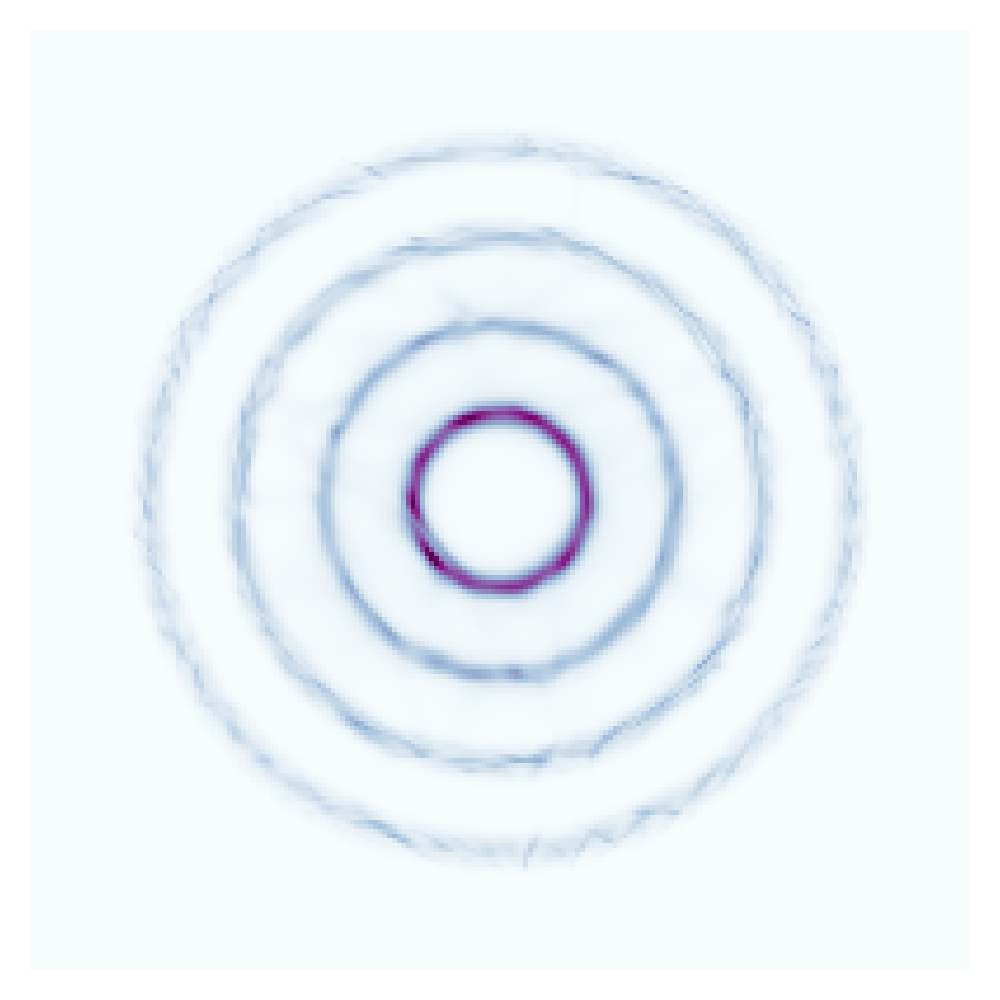}
    \put(-40,-5){\small CP-Flow}
    \caption{Learning toy densities.}
    \label{fig:toy}
\end{center}
\end{wrapfigure}
Having distributional universality for a single flow layer means that we can achieve high expressiveness without composing too many flows. 
We demonstrate this by fitting the density on some toy examples taken from \citet{papamakarios2017masked} and \citet{behrmann2019invertible}.
We compare with the masked autoregressive flow (MAF, \citet{papamakarios2017masked}) and the neural autoregressive flow (NAF, \citep{huang2018neural}).
Results are presented in \cref{fig:toy}. 
We try to match the network size for each data.
All models fit the first data well.
As affine couplings cannot split probability mass, MAF fails to fit to the second and third datasets\footnote{\citet{behrmann2019invertible} demonstrates one can potentially improve the affine coupling models by composing many flow layers. But here we restrict the number of flow layers to be 3 or 5.}. 
Although the last dataset is intrinsically harder to fit (as NAF, another universal density model, also fails to fit it well), the proposed method still manages to learn the correct density with high fidelity.

\subsection{Approximating optimal coupling}
\begin{figure}
    \centering
    \includegraphics[width=0.24\textwidth]{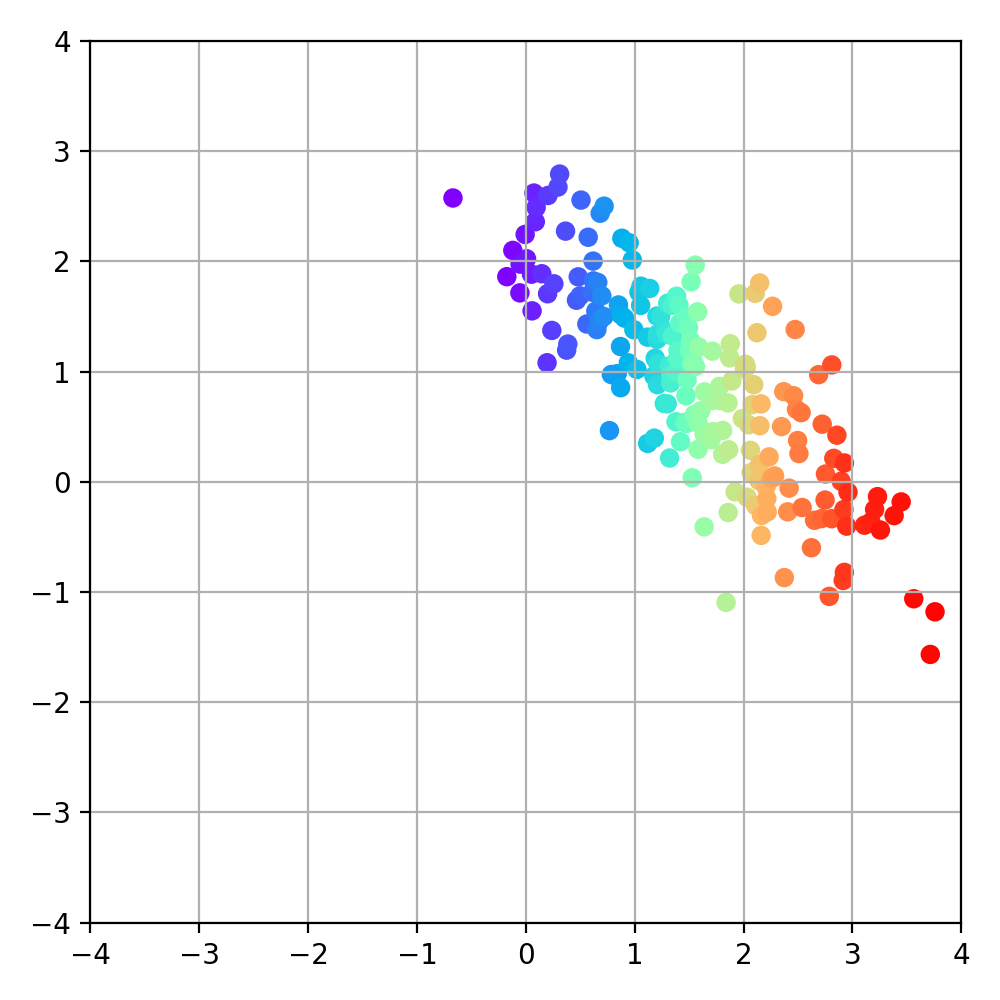}
    \put(-35,95){data: $x$}
    \hfill
    \includegraphics[width=0.24\textwidth]{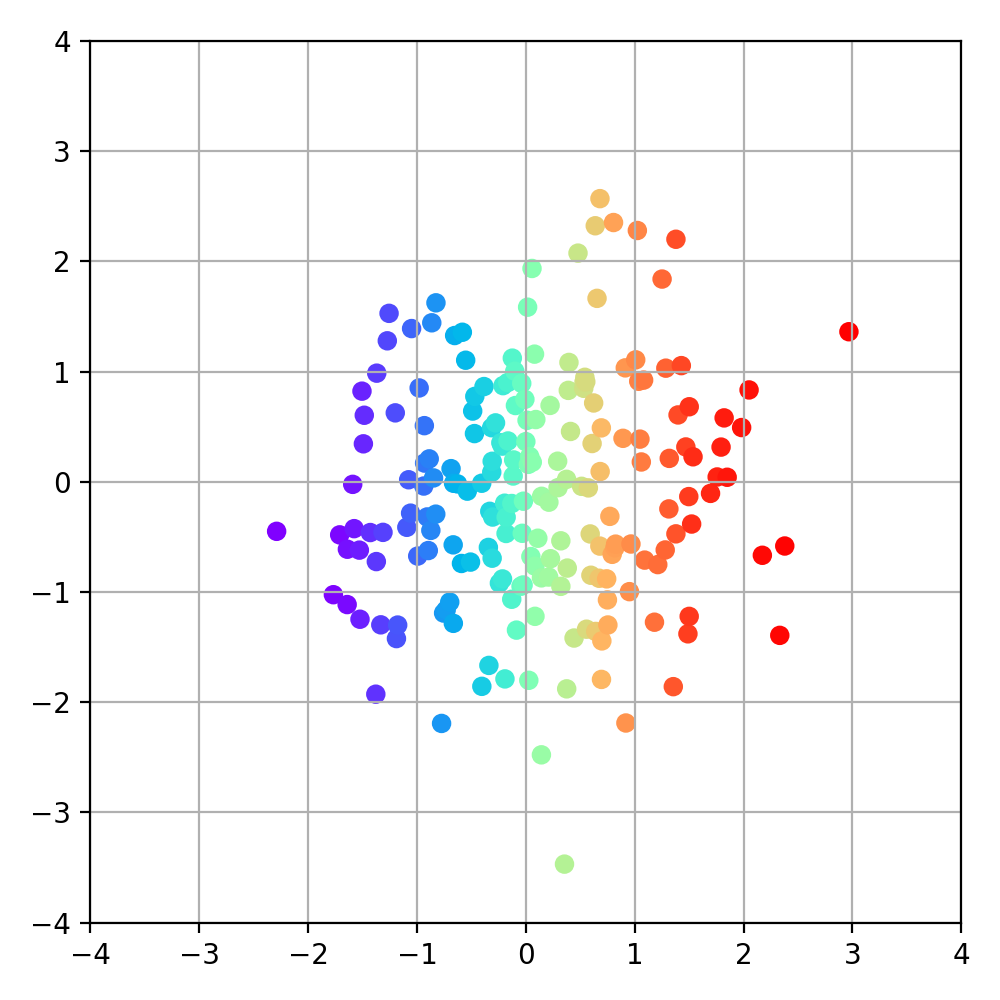}
    \put(-55,95){$z=f_{iaf}(x)$}
    \hfill
    \includegraphics[width=0.24\textwidth]{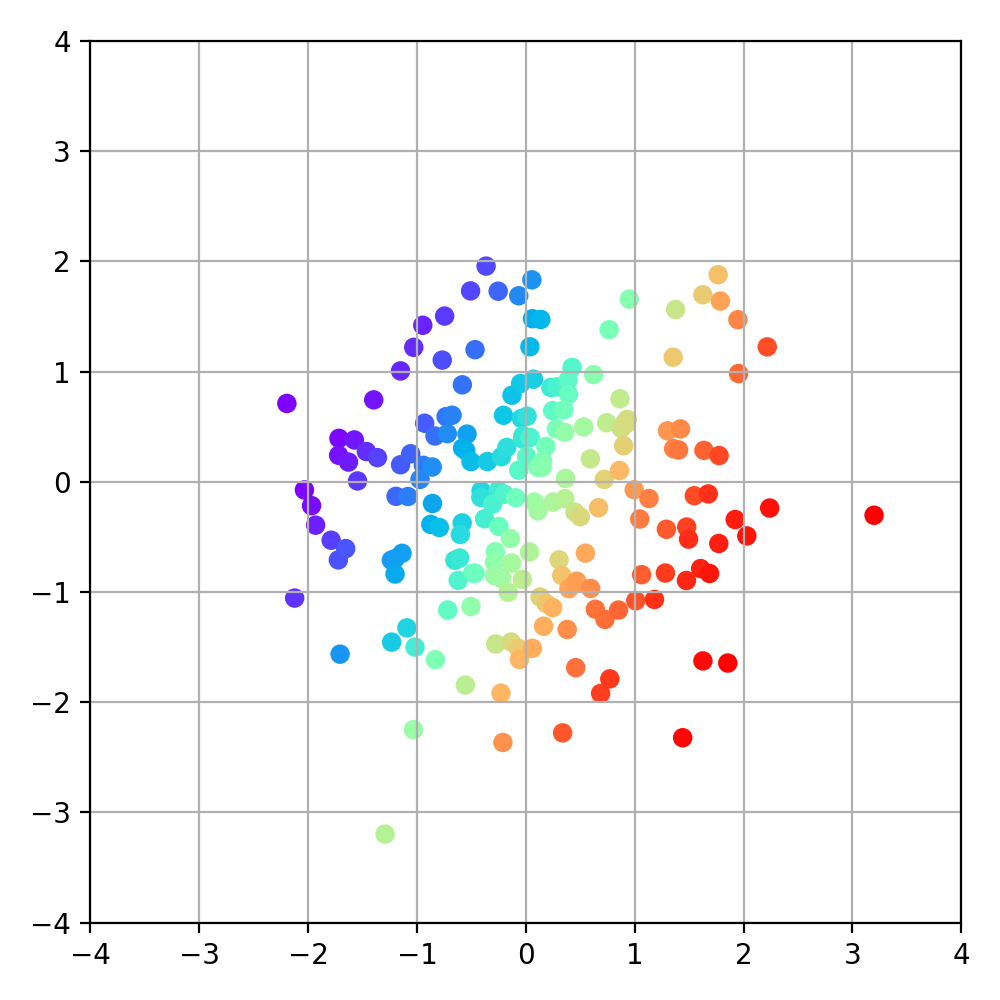}
    \put(-50,95){$z=f_{cp}(x)$}
    \hfill
    \begin{subfigure}{0.24\textwidth}
    \vspace{-10em}
    \includegraphics[width=1.0\textwidth]{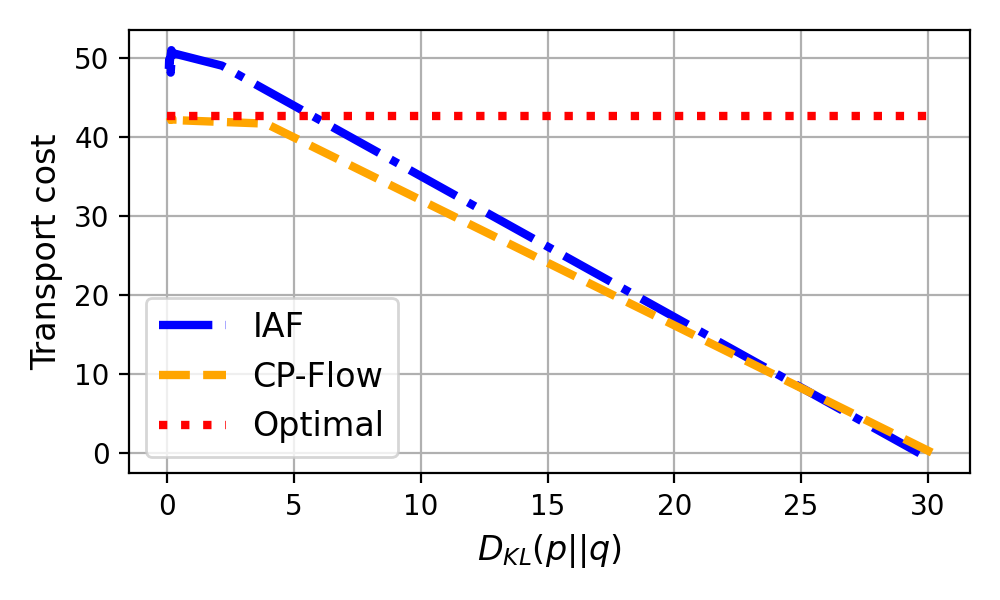}
    \put(-30,35){$d=8$}
    \vfill
    \includegraphics[width=1.0\textwidth]{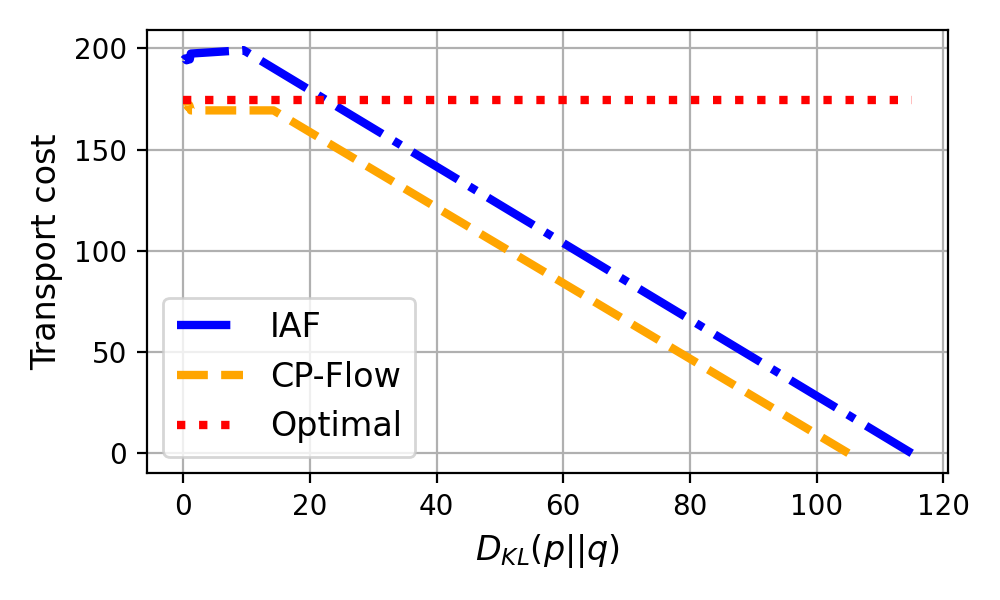}
    \put(-35,35){$d=16$}
    \end{subfigure}
    \caption{\small Approximating optimal transport map via maximum likelihood (minimizing KL divergence). 
    In the first figure on the left we show the data in 2 dimensions.
    The datapoints are colored according to their horizontal values ($x_1$). 
    The flows $f_{iaf}$ and $f_{cp}$ are trained to transform the data into a standard Gaussian prior. 
    In the figures on the right, we plot the expected quadratic transportation cost versus the KL divergence for different numbers of dimensionality. 
    During training the KL is minimized, so the curves read from the right to the left. 
    }
    \label{fig:ot}
\end{figure}
As predicted by Theorem~\ref{thm:optimal}, CP-Flow is guaranteed to converge to the optimal coupling minimizing the expected quadratic cost. 
We empirically verify it by learning the Gaussian density and comparing the expected quadratic distance between the input and output of the flow against $J_{{||x-y||}^2}$
between the Gaussian data and the standard Gaussian prior (as there is a closed-form expression). 
In \cref{fig:ot}, we see that the transport cost gets closer to the optimal value when the learned density approaches the data distribution (measured by the KL divergence). 
We compare against the linear inverse autoregressive flow~\citep{kingma2016improved}, which has the capacity to represent the multivariate Gaussian density, yet it does not learn the optimal coupling.

\subsection{Density estimation}
\begin{table}[]
    \centering
    \small
    \ra{1.1}
    \begin{tabular}{@{} lrrrrr @{}}
        \toprule
        Model & \textbf{\textsc{Power}} & \textbf{\textsc{Gas}} & \textbf{\textsc{Hepmass}} & \textbf{\textsc{Miniboone}} & \textbf{\textsc{BSDS300}} \\
        \cmidrule{1-6}
        Real NVP~{\small\citep{dinh2017density}} & -0.17 & -8.33 & 18.71 & 13.55 & -153.28 \\
        FFJORD~{\small\citep{grathwohl2018ffjord}} & -0.46 & -8.59 & 14.92 & 10.43 & -157.40 \\
        MADE~{\small\citep{pmlr-v37-germain15}} & 3.08 & -3.56 & 20.98 & 15.59 & -148.85 \\
        MAF~{\small\citep{papamakarios2017masked}} & -0.24 & -10.08 & 17.70 & 11.75 & -155.69 \\
        TAN~{\small\citep{pmlr-v80-oliva18a}} & -0.48 & -11.19 & 15.12 & 11.01 & -157.03 \\
        NAF~{\small\citep{huang2018neural}} & -0.62 & -11.96 & 15.09 & 8.86 & -157.73 \\
        \cmidrule{1-6}
        CP-Flow (Ours) & -0.52 & -10.36 & 16.93 & 10.58 & -154.99\\

        \bottomrule
    \end{tabular}
    \caption{Average test negative log-likelihood (in nats) of tabular datasets in~\citet{papamakarios2017masked} for density estimation models (lower is better). 
    Standard deviation is presented in the \cref{app:tabular}.}
    \label{tab:tab_results}
\end{table}

We demonstrate the efficacy of our model and the proposed gradient estimator by performing density estimation on the standard benchmarks.

\paragraph{Tabular Data} We use the datasets preprocessed by \citet{papamakarios2017masked}.
In 
\cref{tab:tab_results}, we report average negative log-likelihood estimates evaluated on held-out test sets, 
for the best hyperparameters found via grid search.
The search was focused on the number of flow blocks, the width and depth of 
the ICNN potentials. 
See \cref{app:tabular} for details.
Our models perform competitively against
alternative approaches in the literature. 
We also perform an ablation on the CG error tolerance and ICNN architectures in \cref{app:ablation}. 

\paragraph{Image Data} 
Next, we apply CP-Flow to model the density of standard image datasets, MNIST and CIFAR-10. For this, we use convolutional layers in place of fully connected layers. Prior works have had to use large architectures, with many flow blocks composed together, resulting in a large number of parameters to optimize. While we also compose multiple blocks of CP-Flows, we find that CP-Flow can perform relatively well with fewer number of parameters (\cref{tab:img_results}). Notably, we achieve comparable bits per dimension to Neural Spline Flows~\citep{durkan2019neural}---another work promoting fewer parameters---while having using around 16\% number of parameters. 

As prior works use different architectures with widely varying hyperparameters, we perform a more careful ablation study using coupling~\citep{dinh2014nice,dinh2017density} and invertible residual blocks~\citep{chen2019residual}. We replace each of our flow blocks with the corresponding baseline. We find that on CIFAR-10, the baseline flow models do not perform nearly as well as CP-Flow. We believe this may be because CP-Flows are universal with just one flow block, whereas coupling and invertible residual blocks are limited in expressivity or Lipschitz-constrained.

\newcommand{\tcen}[1]{\multicolumn{1}{c}{#1}}
\newcommand{\tnan}{---}
\begin{table}\centering
\small
\ra{1.1}
\begin{tabular}{@{} l r r r r r @{}}\toprule
  & \multicolumn{2}{c}{\textbf{\textsc{MNIST}}} &
  & \multicolumn{2}{c}{\textbf{\textsc{CIFAR-10}}} \\
\cmidrule(lr){2-3} \cmidrule(l){5-6} 
Model & \tcen{Bits/dim} & \tcen{N. params}
& & \tcen{Bits/dim} & \tcen{N. params} \\
\cmidrule(r){1-1}\cmidrule(lr){2-3} \cmidrule(l){5-6} 
Real NVP~{\small\citep{dinh2017density}} & 1.05 & \na & & 3.49 & \na \\
Glow~{\small\citep{kingma2018glow}} & 1.06 & \na & & 3.35 & 44.0M$^\dagger$ \\
RQ-NSF~{\small\citep{durkan2019neural}} & \tnan & \tnan & & 3.38 & 11.8M$^\dagger$ \\
Residual Flow~{\small\citep{chen2019residual}} & 0.97 & 16.6M$^\ddagger$ & & 3.28 & 25.2M$^\ddagger$ \\
\cmidrule(r){1-1}\cmidrule(lr){2-3} \cmidrule(l){5-6} 
Coupling Block Ablation & 1.02 & 3.1M & & 3.58 & 2.9M \\
Residual Block Ablation & 1.04 & 2.9M & & 3.46 & 3.1M \\
CP-Flow (Ours) & 1.02 & 2.9M & & 3.40 & 1.9M \\
\bottomrule
\end{tabular}
\caption{Negative log-likelihood (in bits) on held-out test data (lower is better). $^\dagger$Taken from~\citet{durkan2019neural}. $^\ddagger$Obtained from official open source code.}
\label{tab:img_results}
\end{table}

\subsection{Amortizing ICNN for Variational Inference}

\begin{wraptable}[13]{r}{0.45\textwidth}
\vspace{-1.2em}
\begin{center}
    \centering
    \small
    \ra{1.1}
    \setlength{\tabcolsep}{2pt}
    \resizebox{\linewidth}{!}{%
    \begin{tabular}{@{} lrrr @{}}
        \toprule
        {} & {\textbf{\textsc{\scriptsize Freyfaces}}} & {\textbf{\textsc{\scriptsize Omniglot}}} & {\textbf{\textsc{\scriptsize Caltech}}} \\
        \cmidrule{1-4}
        { Gaussian} & 4.53 & 104.28 & 110.80 \\
        { Planar} & 4.40 & 102.65 & 109.66 \\
        { IAF} & 4.47 & 102.41 & 111.58 \\
        { Sylvester} & 4.45 & 99.00 & 104.62 \\
        \cmidrule{1-4}
        { CP-Flow (vanilla)} & 4.47 & 102.06 & 106.53 \\
        { CP-Flow (aug)} & 4.45 & 100.82 & 105.17 \\
        \bottomrule
    \end{tabular}
    }
    \caption{\small Negative ELBO of VAE (lower is better).
    Standard deviation reported in \cref{app:vae}.}
    \label{tab:vae_results}
\end{center}
\end{wraptable}

Normalizing flows also allow us to employ a larger, more flexible family of distributions for variational inference \citep{rezende2015variational}. 
We replicate the experiment conducted in \citet{van2018sylvester} to enhance the variational autoencoder \citep{kingma2013auto}.
For inference amortization, we use the partially input convex neural network from \citet{amos2017input}, and use the output of the encoder as the additional input for conditioning. 
As \cref{tab:vae_results} shows, the performance of CP-Flow is close to the best reported in \citet{van2018sylvester} without changing the experiment setup. 
This shows that the convex potential parameterization along with the proposed gradient estimator can learn to perform accurate amortized inference. 
Also, we show that replacing the vanilla ICNN with the input-augmented ICNN leads to improvement of the likelihood estimates.

\section{Conclusion}
We propose a new parameterization of normalizing flows using the gradient map of a convex potential. 
We make connections to the optimal transport theory to show that the proposed flow is a universal density model, and leverage tools from convex optimization to enable efficient training and model inversion. 
Experimentally, we show that the proposed method works reasonably well when evaluated on standard benchmarks. 

Furthermore, we demonstrate that the performance can be improved by designing better ICNN architectures.
We leave the exploration for a better ICNN and convolutional ICNN architecture to improve density estimation and generative modeling for future research.

\section*{Acknowledgements}
We would like to acknowledge the Python community \citep{van1995python, oliphant2007python} for developing the tools that enabled this work, including
numpy \citep{oliphant2006guide,van2011numpy, walt2011numpy, harris2020array},
PyTorch \citep{paszke2019pytorch},
Matplotlib \citep{hunter2007matplotlib},
seaborn \citep{seaborn},
pandas \citep{mckinney2012python}, and
SciPy \citep{jones2014scipy}.

\bibliography{ref}
\bibliographystyle{iclr2021_conference}

\newpage
\appendix

\section{Invertibility of CP-Flow}
\label{app:invertibility}

In this section, we formally discuss the invertibility of CP-Flow, and establish the connection to convex conjugate (Legendre-Fenchel transform).  
We work with $C^2$ convex potentials $F:\R^d\rightarrow\R$; i.e. $F$ is convex and twice continuously differentiable.
We first check that $f:=\nabla F$ is injective if $F$ is \emph{strictly} convex. 
This is because if $F$ is twice differentiable and strictly convex, the Hessian matrix $H:=\nabla^2F$ is symmetric positive definite, and thus
$z^\top H z >0$ for any non-zero vector $z$. 
We then have, for any $x\neq y$,
$$f(x) - f(y) = \int_\gamma H(\gamma) d\gamma = \int_0^1 H(y + t (x-y)) (x-y) dt, $$
where we used the \emph{gradient theorem} for the line integral on a path $\gamma$ connecting $x$ and $y$, and substituted $t \mapsto y + t (x-y)$ for $t$ going from 0 to 1. 
Positive-definiteness implies $(x-y)^\top(f(x)-f(y)) > 0$, and since $x\neq y$, $f(x)\neq f(y)$.

Now we further assume $F$ is \emph{strongly} convex.
Then for any $y$, $F_y(x) := F(x) - x^\top y$ is also strongly convex, which, by Taylor's theorem, implies that we can place a quadratic lower bound on $F_y$ and thus $F_y(x)\rightarrow\infty$ whenever $||x||\rightarrow\infty$.
This means for a sufficiently large constant $R$, the sub-level set $S_R:=\{x:F_y(x) \leq R\}$ is non-empty and compact. 
By the \emph{Weierstrass extreme value theorem}, $F_y$ (restricted on $S_R$) has a minimizer $x^*$, and it is also the global minimizer over $\R^d$. 
Now lets differentiate $F_y$ at $x^*$, which gives $\nabla F(x^*) - y$.
The gradient must be equal to $0$ by the first order condition, meaning $x^*$ is the inverse point of $y$ under $f$. 
Since this holds for any $y\in\R^d$, $f$ is surjective.  

Now recall the definition of the convex conjugate: $$F^*(y):=\sup_{x} x^\top y - F(x) = {x^*}^\top y - F(x^*),$$ where $x^*$ found by the above procedure depends on $y$. Note that $x^*$ is differentiable by the inverse function theorem.
Thus, differentiating $F^*$ yields 
$$\nabla_y F^*(y) = (\nabla_y {x^*})^\top y + x^* - (\nabla_y {x^*})^\top \nabla F(x^*) = x^*$$
since $\nabla F(x^*)=y$. 
This means if $y=f(x) = \nabla_x F(x)$, then $x = \nabla_y F^*(y)$; i.e. $\nabla F^* = (\nabla F)^{-1}$.

\section{Softplus Type Activation}
\label{app:softplus}
In this section, we let $r(x)=\max(0, x)$ be the ReLU activation function. 
\begin{definition}
We say a function $s$ is of the softplus type if the following holds
\begin{enumerate}[label=(\alph*)]
\item $s\geq r$
\item $s$ is convex
\item $|s(x)-r(x)|\rightarrow 0$ as $|x|\rightarrow \infty$
\end{enumerate}
\end{definition}

Note that a softplus-type activation function is necessarily continuous, non-decreasing, and uniformly approximating ReLU in the following sense:
$$|s(xa)/a - r(x)|\rightarrow 0$$
uniformly for all $x\in\R$ as $a\rightarrow\infty$.

The following proposition characterizes a big family of softplus-type functions, and establishes a close connection between softplus type functions and probability distribution functions. 

\begin{proposition}
Let $p$ be a probability density function of a random variable with mean zero. 
Then the convolution $s:=p*r$ is a softplus-type function. 
Moreover, $s(x) = \int_{-\infty}^x F_p(y) dy$, where $F_p$ is the distribution function of $p$, and $s$ is at least twice differentiable. 
\end{proposition}
\begin{proof}
We first prove a claim (i): $xF_p(x)\rightarrow0$ as $x\rightarrow-\infty$.
First, for $x\leq 0$, 
$$0\geq xF_p(x) = \int_{-\infty}^x xp(y)dy \geq \int_{-\infty}^x yp(y)dy  $$
Since $1_{y\leq x}yp(y)\rightarrow 0$ as $x\rightarrow-\infty$ and $|1_{y\leq x} y p(y)| \leq |y p(y)|$, which is integrable by assumption, the integral on the RHS of the above goes to $0$ by the dominated convergence theorem. 

We now show the identity.
By definition, since $\int y p(y)dy=0$
\begin{align}
s(x) &= \int \max(x-y, 0) p(y) dy = \int \max(x,y) p(y) dy \nonumber\\
&= \int_{-\infty}^x xp(y)dy + \int_{x}^\infty y p(y) dy
= x F_p(x) + \int_{x}^\infty y p(y) dy
\label{eq:softplus_first}
\end{align}

where we've used claim (i) to evaluate $xF_p(x)$ as $x\rightarrow-\infty$.
On the other hand, integration by part implies
\begin{align}
\int_{-\infty}^x F_p(y)dy
= y F_p(y) \big|_{-\infty}^x - \int_{-\infty}^x yp(y)dy
= x F_p(x) + \int_x^\infty yp(y)dy    
\label{eq:softplus_second}
\end{align}

Twice differentiability follows from the differentiability of $F_p$. 

(a) Now since $r(x)$ is convex, Jensen's inequality gives
$$s(x) =\int r(y) p(x-y) dy 
= \E[r(y)] \geq r(\E[y]) = r(x)
$$

(b) $s$ is convex because $s'=F_p$ is non-decreasing. 

(c) To show that $s$ ane $r$ are asymptotically the same, we notice the integral on the RHS of (\ref{eq:softplus_first}) goes to $0$ as $|x|\rightarrow\infty$ (by the dominated or monotone convergence theorem). 
It suffices to show $xF_p(x)$ goes to $0$ as $x\rightarrow-\infty$, which is just claim (i), and $x-xF_p(x)>0$ goes to $0$ as $x\rightarrow\infty$.
To show the latter, we can rewrite $x-xF_p(x) = x(1-F_p(x)) = \int_x^\infty x p(y) dy$, and the same argument as the claim holds with a vanishing upper bound.

\end{proof}

When $p$ is taken to be the standard logistic density, the corresponding $s=p*r$ is simply the regular softplus activation function.
We list a few other softplus-type functions in \cref{tab:softplus} and visualize them in \cref{fig:softplus}. 
We experimented with the Gaussian-softplus and the logistic-softplus.

\begin{minipage}{0.65\textwidth}
\begin{table}[H]
    \vspace{0.5cm}
    \centering
    \begin{tabular}{ccc}
        \toprule
        & $p(y)$ & $s:=p*r$\\
        \midrule
        Logistic & $ \frac{\exp(-x)}{(1+\exp(-x))^2} $ & $\log \left(1+\exp(x)\right)$ \\
        Laplace & $\frac{e^{-|x|}}{2}$ & $r(x) + \frac{e^{-|x|}}{2}$\\
        Gaussian & $\frac{e^{-\frac{x^2}{2}}}{\sqrt{2\pi}}$ & $\frac{{\sqrt{\frac{\pi}{2}}x \erf\left(\frac{x}{\sqrt{2}}\right) + e^{-\frac{x^2}{2}}+\sqrt{\frac{\pi}{2}}x}}{\sqrt{2\pi}}$ \\
        \bottomrule
    \end{tabular}
    \vspace{0.3cm}
    \caption{Formula of some softplus-type functions.}
    \label{tab:softplus}
\end{table}
\end{minipage}
\begin{minipage}{0.35\textwidth}
\begin{figure}[H]
    \centering
    \includegraphics[width=0.95\textwidth]{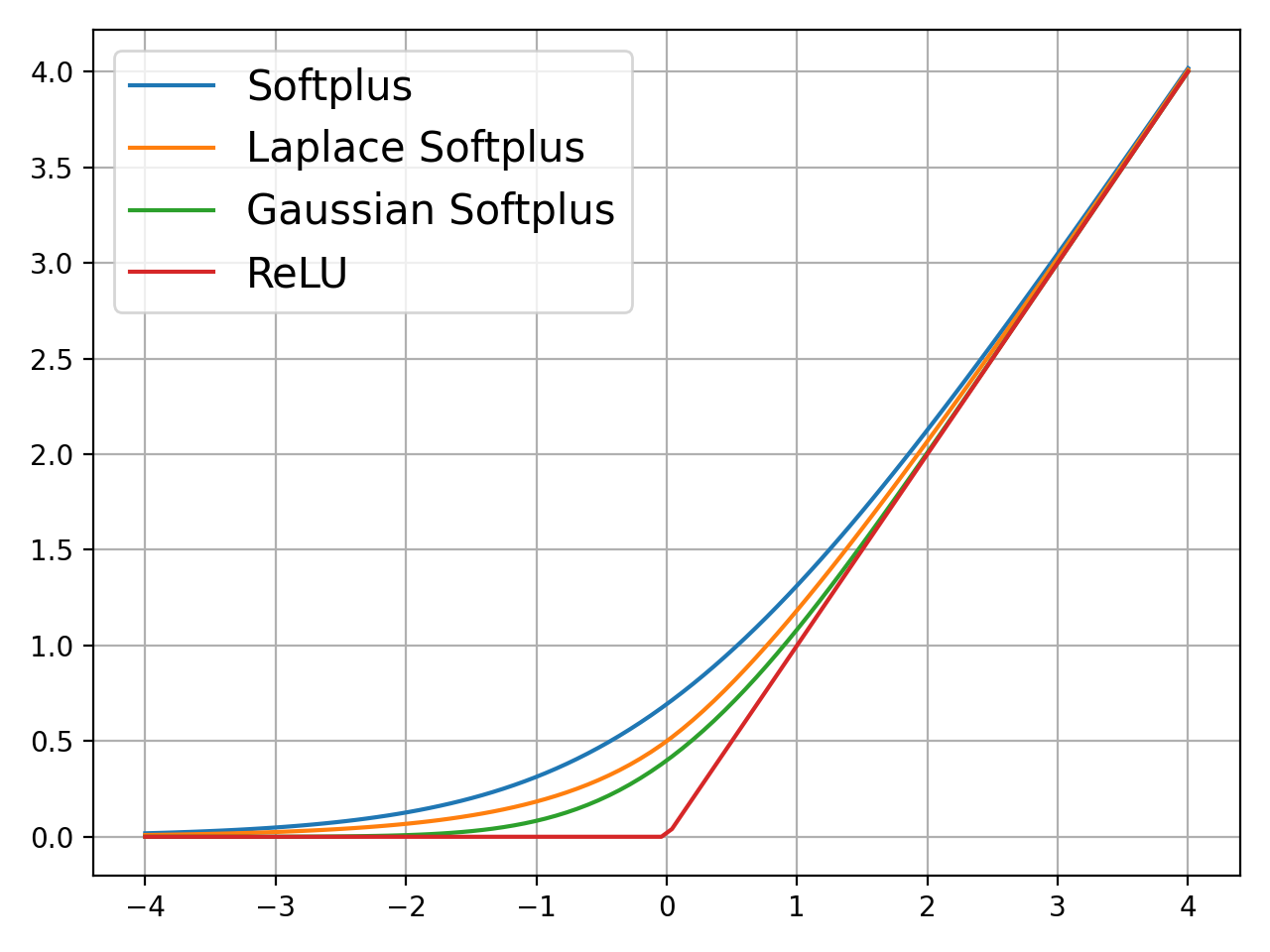}
    \caption{Softplus-type functions.}
    \label{fig:softplus}
\end{figure}
\end{minipage}
\newpage
\section{Universality Proof}
\label{app:univ}
\paragraph{Notation:} Given a convex set $\Omega\subseteq\R^d$,
we let $\gC(\Omega)$ denote the set of continuous functions on $\Omega$, 
and $\gC_{\times}(\Omega):=\{f\in \gC(\Omega): f \text{ is convex}\}$ denote the set of convex, continuous functions.

We first show that ICNNs with a suitable activation function are dense in $C_\times$. 
A similar result can be found in \citet{chen2019optimal}, where they use a different constructive proof: first show that piecewise maximum of affine functions, \ie{} the maxout unit \citep{goodfellow2013maxout}, can approximate any convex function, and then represent maxout using ICNN. 
We emphasize our construction is simpler (see proof of Proposition~\ref{prop:icnn_density_convex}).

The following proposition proves that functions that are pointwise maximum of affine functions, are a dense subset of $C_\times$.

\begin{proposition}
Pointwise maximum of affine functions is dense in $\gC_{\times}([0,1]^d)$.
\label{prop:maxout}
\end{proposition}
\begin{proof}
Fix some $\epsilon >0$. 
Since $f\in \gC_{\times}([0,1]^d)$ is uniformly continuous on $[0,1]^d$, there exists some $\delta>0$ such that $|f(x)-f(y)|<\epsilon$ provided that $||x-y||<\delta$. 
Let $n$ be big enough such that $2^{-n}<\delta$, and let 
$\gX$ be the set of points whose coordinates sit on $i2^{-n}$ for some $1\leq i\leq 2^n-1$ (i.e. there are $|\gX|=(2^n-1)^d$ points in $\gX$).
For each $y\in \gX$, let $L_y(x) := \nabla f(y)^\top(x-y) + f(y)$ be a supporting hyperplane of the graph of $f$, where $\nabla f(y)$ is a subgradient of $f$ evaluated at $y$. 
Then we have a convex approximation $f_\epsilon(x):=\max_{y\in\gX} L_y(x)$ which bounds $f$ from below.
Moreover, letting $y_{|x} := \argmin_{y\in\gX} ||x-y||$, we have (for $x\not\in \gX$),

\begin{align*}
f(x) - f_\epsilon(x)
&= f(x) - \max_{y\in \gX} L_y(x) \\
&\leq f(x) - L_{y_{|x}}(x) \\
&= f(x) - f(y_{|x}) - \sum_{i=1}^d \nabla f(y_{|x})_i ( x_i - y_{|x, i}) \\
&\leq f(x) - f(y_{|x}) + \sum_{i=1}^d |\nabla f(y_{|x})_i | \cdot | x_i - y_{|x, i}| \\
&\leq f(x) - f(y_{|x}) + \sum_{i=1}^d \frac{\epsilon}{| x_i - y_{|x, i}|} \cdot | x_i - y_{|x, i}| \\
&\leq (d+1)\epsilon
\end{align*}

Since $\epsilon$ is arbitrary, this construction forms a sequence of approximations converging uniformly to $f$ from below. 
\end{proof}

The following proposition shows that maxout units can be equivalently represented by ICNN with the ReLU activation, and thus entails the density of the latter (as well as ICNN with softplus activation). 

\begin{proposition}
\label{prop:icnn_density_convex}
ICNN with ReLU or softplus-type activation is dense in  $\gC_{\times}([0,1]^d)$.
\end{proposition}
\begin{proof}
Let $r(x) = \max(0, x)$ be the ReLU activation funciton.
Any convex piecewise linear function $f(x)$ can be represented by
$f(x)=\max(L_1, ... , L_k)$ where $L_j=a_j^\top x + b_j$, which can then be reduced to
\begin{align*}
f(x)
&= r(\max( L_1 - L_k, ..., L_{k-1} - L_k ) ) + L_k \\
&= r(r(\max(L_1-L_{k-1}, ... , L_{k-2} - L_{k-1})) + L_{k-1} - L_k) + L_k \\
&= z_k
\end{align*}
where $z_j := r(z_{j-1}) + L_j'$
for $2\leq j\leq k$, $z_1 = L_1-L_2$,
$L_j' := L_j - L_{j+1}$
for $2\leq j\leq k-1$, and
$L_k':= L_k$.

Since by Proposition~\ref{prop:maxout}, pointwise maximum of affine functions is dense in $\gC_{\times}$, so is ICNN with the ReLU activation function. 
The same holds for softplus since softplus can be used to uniformly approximate ReLU. 

\end{proof}

\diffcvx*
\begin{proof}
We let $x$ be a differentiable point of $G$. 
Since convergence of derivatives wrt each coordinate can be dealt with independently, we assume $d=1$ without loss of generality. 
We can write $f_n$ as 
$$f_n(x) = \lim_{m} f_{nm}(x) \quad\text{ where }\quad f_{nm}(x) = \frac{F_n(x-1/m) - F_n(x)}{-1/m}$$

The problem can be rephrased as proving \footnote{Note that there is an implicit dependency on $x$ since the result is pointwise.}
\begin{align}
\lim_n \lim_m f_{nm} = \lim_m \lim_n f_{nm} 
\label{eq:rephrase_limit_swap}
\end{align}

Note that $f_{nm}$ is non-decreasing in $m$ since $F_n$ is convex, and thus $f_{nm}\leq f_n$.
Since $f_{nm}$ converges to $f_n$, for any $\epsilon>0$, we can find an integer $\mu(\epsilon, n)$ such that for all $m\geq \mu(\epsilon, n)$, 
$|f_{nm} - f_n|\leq\epsilon$.
Let $m_{k}$ be a subsequence of $\{m\geq1\}$ defined as $m_{k}=\mu(2^{-k}, n)$. 


Then $|f_{nm_{k+1}} - f_{nm_k}| \leq 2^{-k}$, which is integrable wrt the counting measure on positive integers $k$, since
$$\int 2^{-k} = \lim_{K\rightarrow\infty}\sum_{k=1}^K 2^{-k} = 1$$

Thus, letting $f_{nm_0}=0$, by the Dominated Convergence Theorem, we have

$$\lim_n \lim_K f_{nm_K} =  \lim_n \int f_{nm_{k}} - f_{nm_{k-1}} = \int \lim_n f_{nm_{k}} - f_{nm_{k-1}} = \lim_K \lim_n f_{nm_K} $$

Although we are only looking at the limit of the subsequence $m_k$, this is sufficient for (\ref{eq:rephrase_limit_swap}), since  the LHS is equal to $\lim_n f_n$ (since each $F_n$ is differentiable), and
by linearity of the limit, the RHS is equal to 
$\lim_K \frac{G(x - 1/m_K) - G(x)}{ -1/m_K} = g(x)$ (since $G$ is differentiable at $x$ by assumption). 

Since the set of points over which $G$ is not differentiable is a set of measure zero \citep[Thm.~25.5]{rockafellar1970convex}, the convergence holds almost everywhere. 
\end{proof}


\univ*
\begin{proof}
Assume $\mu$ and $\nu$ have finite second moments.
Since $\mu$ is absolutely continuous, 
by \textbf{Brenier's theorem}, there exists a convex function $G:\R^d\rightarrow\R$ such that $\nabla G(X)\overset{d}=\nu$ (where the gradient is unique up to changes on a null set).
By Proposition~\ref{prop:icnn_density_convex}, there exists a sequence of ICNN $F_n$ converging to $G$ pointwise everywhere. 
Such a sequence can be found since we can let $F_n$ approximate $G$ with a uniform error of $1/n$ on a compact domain $[-n,n]^d$. 
Theorem~\ref{thm:diff_convex_converge} then implies the gradient map $f_n:=\nabla F_n$ converges to $\nabla G$ pointwise almost everywhere. 
This implies the weak convergence of the pushforward measure of $f_n\circ X$. 

Now remove the finite second moment assumption 
and let $X$ and $Y$ be random variables distributed according to $\mu$ (with Lebesgue density $p$) and $\nu$, respectively. 
Denote by $B_k$ a ball of radius $k>0$ centered at the origin, i.e. $B_k:=\{x:||x||\leq k\}$. 
Let $X_k=X1_{X\in B_k} + U_k1_{X\not\in B_k}$, where $U_k$ is an independent random variable distributed uniformly on $B_k$, and let $\mu_k$ be the law of $X_k$. 
Then $X_k\rightarrow X$ almost surely as $k\rightarrow\infty$, and $\mu_k$ is still absolutely continuous wrt the Lebesgue measure, with its density being $p(x) + \frac{1}{\vol(B_k)}\mu(||X||> k)$ if $||x||\leq k$ or $0$ otherwise. 
Let $Y_k$ and $\nu_k$ be defined similarly (while $\nu_k$ may not be absolutely continuous wrt Lebesgue). 
From above, since $X_k$ and $Y_k$ are bounded and admit a finite second moment, we know fixing $k$, we can find a sequence of $f_{k,n} = \nabla F_{k,n}$ such that $f_{k,n}\circ X_k \rightarrow Y_k$ in distribution as $n\rightarrow \infty$.
Now since weak convergence is metrizable, choose $n_k$ to be large enough such that the distance between the pushforward of $f_{k,n_k} \circ X_k$ and $\nu_k$ is at most $1/k$.
An application of triangle inequality of the weak metric implies $f_{k,n_k}\circ X_k\rightarrow Y$ in distribution as $k\rightarrow\infty$. 
Finally, note that since $\{||f_{k,n_k} \circ X_k - f_{k,n_k} \circ X || > \epsilon\} \subseteq \{||X||>k\}$, by monotonicity, we have
\begin{align}
\sP\left(\limsup_k ||f_{k,n_k} \circ X_k - f_{k,n_k} \circ X || > \epsilon \right) \leq \sP\left(\limsup_k ||X|| > k\right) = 0
\end{align}
and thus $||f_{k,n_k} \circ X_k - f_{k,n_k} \circ X ||\rightarrow 0$ almost surely (where $\sP$ is the underlying probability measure of the measure space). 
By Lemma 1 of \citet{huang2020solving} (or equivalently Lemma 2 of \citet{huang2020augmented}), $f_{k,n_k} \circ X$ converges in distribution to $Y$.

\end{proof}

\section{Optimality Proof}
\label{app:optim}
\optim*
The result can be deduced from the fact that optimality is ``stable'' under weak limit; see for example, \citet[Thm 1.50]{santambrogio2015optimal}.  
We prove the special case of quadratic cost function. 
\begin{proof}
We claim that if $F$ is a convex potential such that $Z=\nabla F(X)$ has $\nu$ as its law, then $\nabla F \equiv \nabla G$ almost surely. 
The proof of the claim is originally due to \citet{ruschendorf1990characterization}, but we present it here for completeness. 
Let $Z'$ be another random variable distributed by $\nu$. 
Then by the Fenchel-Young inequality (applied to the convex potential $F$),
$$\E[X^\top Z'] 
\leq \E[F(X) + F^*(Z')] 
= \E[F(X) + F^*(Z)] 
= \E[F(X) + F^*(\nabla F(X))] 
= \E[X^\top \nabla F(X)]$$
This concludes the proof since $\nabla G$ uniquely solves the transportation problem, which is equivalent to finding a transport map $\tilde{g}$ that maximizes the covariance:
$$\E[||X-\tilde{g}(X)||^2]
= \E[||X||^2+||\tilde{g}(X)||^2] - 2 \E[X^\top \tilde{g}(X)] 
$$
Let $F_\infty$ to be the pointwise limit of $F_n$.
Then for any $x_1$, $x_2$ and $t\in[0,1]$,
\begin{align*}
F_\infty(tx_1+(1-t)x_2) 
&= \lim_{n\rightarrow\infty} F_n(tx_1+(1-t)x_2) \\
&\leq \lim_{n\rightarrow\infty} tF_n(x_1) + (1-t)F_n(x_2) \\
&=  \lim_{n\rightarrow\infty} tF_n(x_1) +  \lim_{n\rightarrow\infty} (1-t)F_n(x_2)
= tF_\infty(x_1) + (1-t)F_\infty(x_2)
\end{align*}
That is, $F_\infty$ is convex. 
Now since $F_n$ is a convergent sequence of convex functions, its gradient $\nabla F_n$ also converges pointwise almost everywhere to $\nabla F_\infty$ by Theorem \ref{thm:diff_convex_converge}. 
Let $\rho$ denote the Prokhorov metric, which metrizes the weak convergence, and by abuse of notation, we write $\rho(X,Y)$ to denote the distance between the law of $X$ and $Y$.
Then
$$\rho(\nabla F_\infty(X), Y) \leq
\rho(\nabla F_\infty(X), \nabla F_n(X)) + \rho(\nabla F_n(X), Y)
$$
which means $\nabla F_\infty(X)$ and $Y$ have the same law, $\nu$.
Then by the claim, $\nabla F_\infty \equiv \nabla G$, and thus $\nabla F_n\rightarrow\nabla G$ a.s. as $n\rightarrow\infty$.

\end{proof}

\newpage
\section{Experimental Details}
\label{app:exp}

\subsection{Architecture details}
\paragraph{Initialization} As ICNNs have positive weights, its initialization has a different dynamics than a standard feed-forward network. 
If not stated otherwise, all parameters are initialized using standard PyTorch modules \citep{paszke2019pytorch}.
To parameterize positive weights, we modify the weight parameters of a standard linear layer with the softplus activation. 
We then divide all the weights by the total number of incoming units, so that the average magnitude of each hidden unit will not grow as the dimensionality of the previous layer increases. 

In addition, we reparameterize the CP-Flow $F_\alpha$ as $F_{w_0,w_1}$ defined as 
$$F_{w_0,w_1} = s(w_0){||x||^2}/{2} + s(w_1) F(x) $$ 
where $s$ is the regular (logistic) softplus. 
$w_0$ is initialized to be $s^{-1}(1)$, and $w_1$ is initialized to be $0$ so $F_{w_0,w_1}$ is closer to the identity map. 

Finally, we insert the ActNorm layer \citep{kingma2018glow} everywhere before an activation function is applied with the data-dependent initialization. 

\paragraph{Activation function}
We use the following convexity-preserving operators when designing an ICNN:
\begin{enumerate}
\item invariance under affine maps: $g\circ f$ is convex if $f$ is linear and $g$ is convex
\item non-negative weighted sums: $\sum_j w_j f_j$ is convex if $w_j$s are non-negative and $f_j$s are convex
\item composing with non-decreasing convex functions: $g\circ f$ is convex if $f$ and $g$ are both convex and $g$ is non-decreasing
\end{enumerate}

Notably, in 1., we do not require $g$ to be non-decreasing. 
We experiment with a symmetrized version of softplus $g(x) = s(x) - 0.5 x$ where $s$ is a softplus-type activation, whenever $g$ is used as the first activation.
This way, the derivative of $g$ is $s'(x) - 0.5$, which ranges between $\pm0.5$ and behaves more like tanh (than sigmoid). 
This can be used for the first hidden layer of a regular ICNN, or the augmented layer of the input-augmented ICNN. 

We also experiment with an offset version of softplus, which is defined as $g(x) = s(x) - s(0)$. 
This way the output of the softplus is more symmetric since it can be negative.

\subsection{Toy examples}
For toy examples, we compute the log-determinant of the Jacobian in a bruteforce manner. 
We use the Adam optimizer with an initial learning rate of 0.005.
We create a data set following the toy distribution of size 50000, and train each model for 50 epochs with a minibatch size of 128. 
For MAF and NAF we cap the gradient norm to be 10 for stability. 
For CP-Flow, we use the Gaussian softplus as activation, and symmetrize it at the first layers. 
\begin{table}[H]
    \centering
    \begin{tabular}{ccccc}
        \toprule
        Data & \textsc{N. flows} & \textsc{N. hidden layers} & \textsc{N. hidden units} \\
        \midrule
        One moon & 5 & 3 & 32 \\
        Eight Gaussians & 5 & 3 & 32   \\
        Rings & 5 & 5 & 256  \\
        \bottomrule
    \end{tabular}
    \caption{Architectural details for toy density estimation.}
    \label{tab:arch_toy}
\end{table}

\subsection{Approximating Optimal Coupling}
\label{app:approx_optim}
For the OT map approximation experiment, we simulate data (of size $50,000$) from a Gaussian distribution $\gN(\mu,\Sigma)$ with a prior $\mu\sim \gN(0,I)$ and $\Sigma\sim \gW(I, d+1)$, where $\gW$ is the Wishart distribution. 
For CP-Flow, we use a network of 5 hidden layers of $64$ hidden units, with the Gaussian-softplus and zero-offset. 
We use the Adam optimizer with a minibatch size of $128$, trained for two epochs to generate the figures.

\subsection{Density Estimation}
\label{app:tabular}

\begin{table}[H]
    \centering
    \begin{tabular}{lccccc}
        \toprule
        Model & \textbf{\textsc{Power}} & \textbf{\textsc{Gas}} & \textbf{\textsc{Hepmass}} & \textbf{\textsc{Miniboone}} & \textbf{\textsc{BSDS300}} \\
        \midrule
        Real NVP & -0.17$\,\pm\,$0.01 & -8.33$\,\pm\,$0.14 & 18.71$\,\pm\,$0.02 & 13.55$\,\pm\,$0.49 & -153.28$\,\pm\,$1.78 \\
        Glow & -0.17$\,\pm\,$0.01 & -8.15$\,\pm\,$0.40 & 18.92$\,\pm\,$0.08 & 11.35$\,\pm\,$0.07 & -155.07$\,\pm\,$0.03 \\
        FFJORD & -0.46$\,\pm\,$0.01 & -8.59$\,\pm\,$0.12 & 14.92$\,\pm\,$0.08 & 10.43$\,\pm\,$0.04 & -157.40$\,\pm\,$0.19 \\
        MADE & 3.08$\,\pm\,$0.03 & -3.56$\,\pm\,$0.04 & 20.98$\,\pm\,$0.02 & 15.59$\,\pm\,$0.50 & -148.85$\,\pm\,$0.28 \\
        MAF & -0.24$\,\pm\,$0.01 & -10.08$\,\pm\,$0.02 & 17.70$\,\pm\,$0.02 & 11.75$\,\pm\,$0.44 & -155.69$\,\pm\,$0.28 \\
        TAN & -0.48$\,\pm\,$0.01 & -11.19$\,\pm\,$0.02 & 15.12$\,\pm\,$0.02 & 11.01$\,\pm\,$0.48 & -157.03$\,\pm\,$0.07 \\
        NAF & -0.62$\,\pm\,$0.01 & -11.96$\,\pm\,$0.33 & 15.09$\,\pm\,$0.40 & 8.86$\,\pm\,$0.15 & -157.73$\,\pm\,$0.04 \\
        CP-Flow & -0.52$\,\pm\,$0.01 & -10.36$\,\pm\,$0.03 & 16.93$\,\pm\,$0.08 & 10.58$\,\pm\,$0.07 & -154.99$\,\pm\,$0.08 \\
        \bottomrule
    \end{tabular}
    \caption{Test negative log-likelihood (in nats) of tabular datasets in~\citet{papamakarios2017masked} for density estimation models (lower is better). Results for compared models are taken from~\citet{grathwohl2018ffjord}. Average and standard deviation report over 3 random seeds at the best hyperparameters found by grid search.
    For CP-Flow trained on \textsc{Gas} and \textsc{BSDS300}, only one seed converged at the time of submission, so we report N/A for the standard deviation. 
    }
    \label{app:tab_results}
\end{table}

\begin{table}[H]
    \centering
    \begin{tabular}{lcccc}
        \toprule
        Dataset & \textsc{N. flows} & \textsc{N. hidden layers} & \textsc{N. hidden units} & \textsc{N. Parameters} \\
        \midrule
        \textbf{\textsc{Power}} & 10 & 5 & 512 & 5,463,272 \\
        \textbf{\textsc{Gas}} & 5 & 5 & 512 & 2,757,276  \\
        \textbf{\textsc{Hepmass}} & 5 & 5 & 512 & 2,923,897  \\
        \textbf{\textsc{Miniboone}} & 2 & 5 & 256 & 379,232 \\
        \textbf{\textsc{BSDS300}} & 10 & 5 & 256 & 2,152,456 \\
        \bottomrule
    \end{tabular}
    \caption{Best hyperparameters found in our search and the consequent total number of model parameters.}
    \label{app:tab_hypers}
\end{table}

For each dataset, we search via grid search
for the best hyperparameter configuration by calculating the negative log-likelihood in the validation set using an 
exact bruteforce computation of the log-determinant of the Jacobian. 
Initially, we focus our search at basic hyperparameters that influence 
the total number of parameters of the model, such as the number of flow
blocks, the number of hidden layers of each convex potential block, and the number of hidden units per layer of 
each block. After a first selection of candidate hyperparameters after 
a constant number of training steps, we instantiate extra
experiments considering variations in the softplus-type activations
we use. We find that Gaussian-softplus and symmetrizing it in the 
activations of the first layer help to the performance,
oftentimes accelerating the training. In \cref{app:tab_hypers}, we
report the final hyperparameter combinations we have used for the
results presented in \cref{app:tab_results}.

\subsection{Generative Modeling with Convolutional ICNN}
\label{app:generative}

On MNIST and CIFAR-10, we used a multiscale architecture. For MNIST, we had 8 CP-Flow blocks, followed by an invertible downsampling~\citep{dinh2017density}, followed by 8 CP-Flow blocks, another downsampling, and final 8 CP-Flow blocks. Before every block was an ActNorm~\citep{kingma2018glow} layer. For CIFAR-10, we had 2 CP-Flow blocks, followed by an invertible downsampling~\citep{dinh2017density}, followed by 2 CP-Flow blocks, downsampling, 2 CP-Flow blocks, downsampling, and final 2 CP-Flow blocks. Before every CP-Flow block was an ActNorm~\citep{kingma2018glow} layer. All ICNN architectures had 4 hidden layers and 64 hidden units wide. We averaged across the final spatial dimensions to obtain a scalar output for the convex potential.

We test the invertibility of a CP-Flow model trained on CIFAR-10 on a set of out-of-distribution data sets constructed by \citet{behrmann2020understanding} in Table \ref{tab:recon_errs}.
Notably, we do not suffer from the exploding inverse problem and can reliably invert all data sets, either in or out of distribution.
Samples of reconstructed images are shown in Figure \ref{fig:reconstructions} which show no visual difference between the original images and their reconstructions.

\begin{figure}
    \centering
    \begin{subfigure}[b]{0.4\linewidth}
    \centering
    \includegraphics[width=\linewidth]{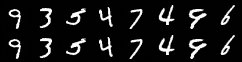}
    \end{subfigure}
    \begin{subfigure}[b]{0.4\linewidth}
    \centering
    \includegraphics[width=\linewidth]{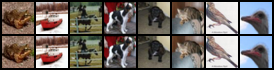}
    \end{subfigure}
    \caption{(\textit{top}) Data samples. (\textit{bottom}) Reconstruction after passing it through a CP-Flow.}
    \label{fig:reconstructions}
\end{figure}

\begin{table}
    \centering
    \begin{tabular}{lccc}
    \toprule
    \textbf{\textsc{Dataset}} & \textbf{\textsc{Glow}} & \textbf{\textsc{ResFlow}} & \textbf{\textsc{CP-Flow}} \\
    \cmidrule{1-4}
    CIFAR-10 (in-dist)                 & 1.12E-6 & 5.16E-4 & 1.96E-3  \\
    \cmidrule{1-4}
    Uniform                            & \texttt{Inf} & 3.04E-4 & 3.62E-3  \\
    Gaussian                           & \texttt{Inf} & 1.31E-4 & 5.97E-3  \\
    Rademacher                         & \texttt{Inf} & 3.43E-5 & 5.27E-3 \\
    SVHN~\citep{netzer2011reading}     & 9.94E-7 & 1.31E-3 & 2.53E-3  \\
    Texture~\citep{cimpoi14describing} & \texttt{Inf} & 3.66E-4 & 1.69E-3  \\
    Places~\citep{zhou2017places}      & \texttt{Inf} & 5.31E-4 & 1.76E-3  \\
    tinyImageNet                       & \texttt{Inf} & 6.26E-4 & 1.65E-3  \\
    \bottomrule
    \end{tabular}
    \caption{Reconstruction RMSE. Results for Glow and ResFlow are taken from \citet{behrmann2020understanding}.}
    \label{tab:recon_errs}
\end{table}

\subsection{Amortizing ICNN for Variational Infernece}
\label{app:vae}
We use the partially input convex neural network from \citet{amos2017input} with multiplicative conditioning. 
There are some other options for conditioning, such as with the hypernetwork \citep{ha2016hypernetworks} or feature-wise transformation \citep{dumoulin2018feature}. 
We remove the non-linear path of the conditioned variable beyong the first layer, since we found that it hurts training of the VAE. 
We use the Gaussian-softplus for all the experiments, symmetrizing it on the first layers. 
For the Caltech experiment, we also use the offset version, and the softplus is initialized with a multiplicative constant of 2, which we found leads to faster convergence.

\begin{table}[H]
    \centering
        \begin{tabular}{ccccc}
        \toprule
        Data & \textsc{N. flows} & \textsc{N. hidden layers} & \textsc{N. hidden units} \\
        \midrule
        FreyFaces & 4 & 4 & 256 \\
        Omniglot & 2 & 2 & 512  \\
        Caltech & 8 & 4 & 256  \\
        \bottomrule
    \end{tabular}
    \caption{Architectural details for the VAE experiment.}
    \label{tab:arch_vae}
\end{table}

\begin{table}[H]
    \centering
    \ra{1.1}
    \begin{tabular}{@{} lrrr @{}}
        \toprule
        Model & \textbf{\textsc{Freyfaces}} & \textbf{\textsc{Omniglot}} & \textbf{\textsc{Caltech}} \\
        \cmidrule{1-4}
        No flow~{\small\citep{kingma2013auto}} & $4.53\pm0.02$ & $104.28\pm0.39$ & $110.80\pm0.74$ \\
        Planar~{\small\citep{rezende2015variational}} & $4.40\pm0.06$ & $102.65\pm0.42$ & $109.66\pm0.42$ \\
        IAF~{\small\citep{kingma2016improved}} & $4.47\pm0.05$ & $102.41\pm0.04$ & $111.58\pm0.38$ \\
        Sylvester~{\small\citep{van2018sylvester}} & $4.45\pm0.04$ & $99.00\pm0.04$ & $104.62\pm0.29$ \\
        \cmidrule{1-4}
        CP-Flow (Ours) & $4.47\pm0.02$ & $102.06\pm0.03$ & $106.53\pm0.55$ \\
        CP-Flow aug (Ours) & $4.45\pm0.03$ & $100.82\pm0.30$ & $105.17\pm0.57$ \\
        \bottomrule
    \end{tabular}
    \caption{\small Negative ELBO of VAE (lower is better).
    For \textsc{Freyfaces} the results are in bits per dim. 
    The numbers are averaged over three runs of experiments. Standard deviation is presented in the \cref{app:vae}.
    For CP-Flow with input-augmented ICNN trained on \textsc{Omniglot}, only one seed converged at the time of submission, so we report N/A for the standard deviation. 
    }
    \label{tab:vae_results_full}
\end{table}

\newpage
\section{Additional Ablation}
\label{app:ablation}

\begin{wrapfigure}[16]{r}{0.5\textwidth}
\vspace{-1.5cm}
\begin{center}
    \includegraphics[width=0.5\textwidth]{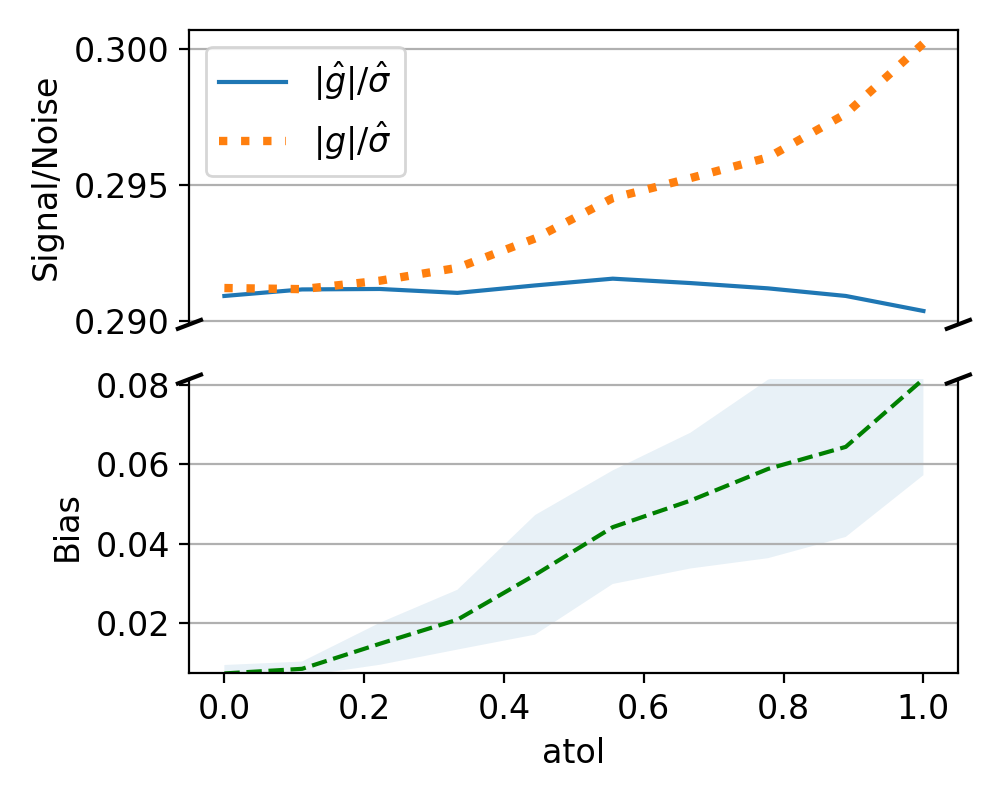}
    \caption{\small Simulated signal-to-noise ratio and bias of the gradient estimator for different atol values}
    \label{fig:bias_s2n}
\end{center}
\end{wrapfigure}
We perform an analysis on the effect of changing the absolute error tolerance level (atol) on the gradient estimator's bias and variance.
In a synthetic setting, we sample a $d\times d$ (where $d=10$) positive definite matrix $H$ from $\gW(I, 2d)$ and scale it down by $1/d$ (so that the diagonal entries do not grow as $d$ increases).
We linearly interpolate atol values between $0$ and $1$, and compare the error statistics of the gradient estimator against the ground truth gradient of $\log\det H$ w.r.t $H$. 
The error statistics are (1) the signal-to-noise ratio (expected value of the estimator, or the ground truth gradient, divided by the standard deviation), and (2) the absolute bias (expected value of the absolute difference between the estimate and the ground truth). 
Since these quantities are intractable, we use Monte Carlo estimate of 100K samples (draw 100K i.i.d. samples of the gradient estimator) to smooth out the curves. 
The estimates are averaged across all entries of $H$ and $10$ different random $H$'s.
Results are presented in \cref{fig:bias_s2n}.

On a more realistic setting, we monitor the average number of CG iterates (hvp calls), per-iteration time, as well as validation loss on the Miniboone dataset. 
Meanwhile, we compare the vanilla ICNN, input-augmented ICNN, as well as a dense version of ICNN \citep{huang2017densely}. 
Figure \ref{fig:ablate_atol_arch} shows that as tolerance value becomes smaller, we indeed need more hvp calls, which eventually saturates at the dimensionality of the data. 
However, there is not much difference in terms of log-likelihood if the tolerance is sufficiently small: the validation loss oscillates and diverges for $\textnormal{atol}=0.1$, while the curves for different atol values smaller than $0.001$ are almost indistinguishable and fairly stable.
On the other hand, there is a noticeable difference in performance if we simply replace the vanilla ICNN with the input-augmented ICNN or the dense ICNN. 
This suggests \emph{tuning the architecture of the convex potential is more crucial for improving the overall performance}. This set of ablation studies were all performed on P100 NVIDIA GPUs and with a constant batch size of 1024.

\paragraph{Runtime of directly backpropogating through Lanczos}
Furthermore, we include an analysis on the effect of the number of eigenvalues used in the stochastic Lanczos quadrature method (\ie{} size of the tridiagonal matrix) on training time. 
In this experiment we directly backpropagated through the Lanczos estimate to compute the stochastic gradient (instead of using the proposed gradient estimator). 
Figure~\ref{fig:ablate_m_lanczos} shows that the runtime is much higher than the runtime of the proposed gradient estimator using CG (bottom left of Figure~\ref{fig:ablate_atol_arch}). 
We note that the experiments with $m=5$ diverged, possibly due to the error in estimation. 
This complements the memory profile shown in Figure~\ref{fig:memory_profile}, and accentuates the lower runtime and memory requirement of the proposed method by contrast.

\begin{figure}
    \centering
    \includegraphics[width=0.49\textwidth]{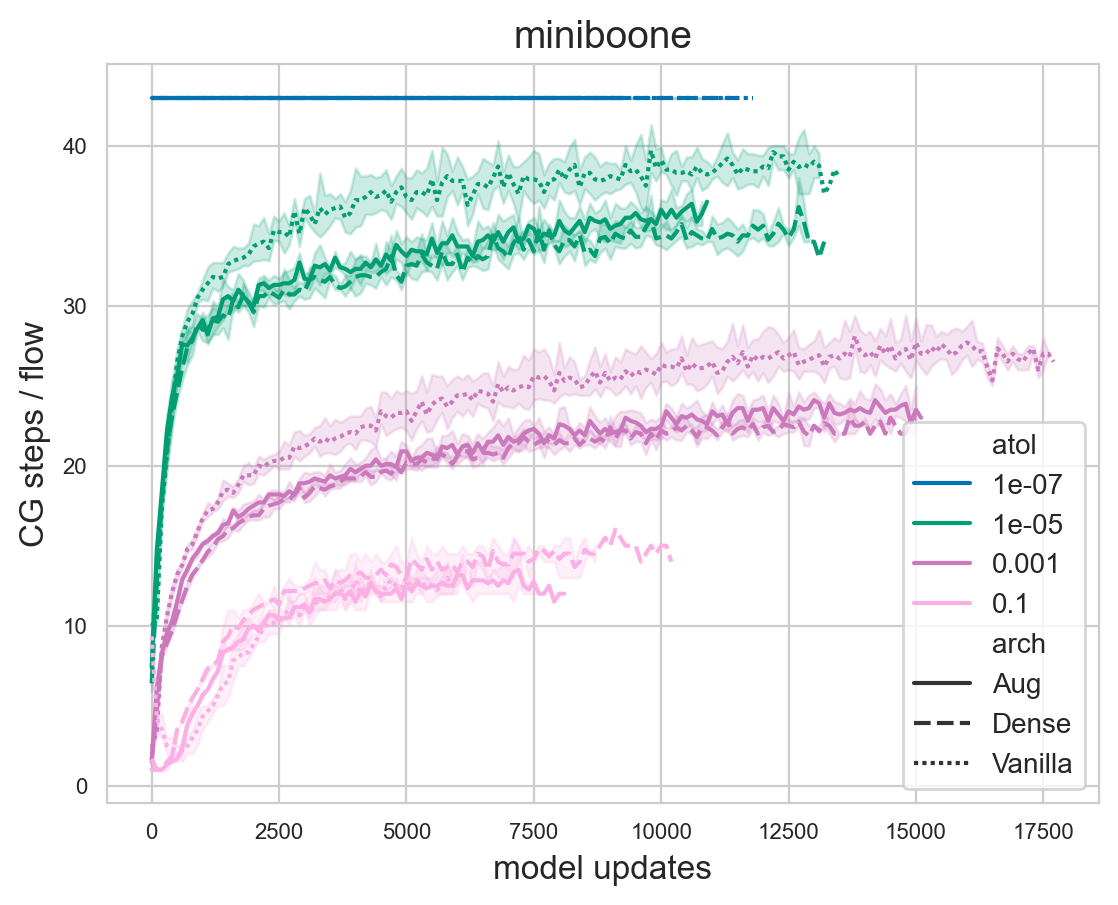}
    \includegraphics[width=0.49\textwidth]{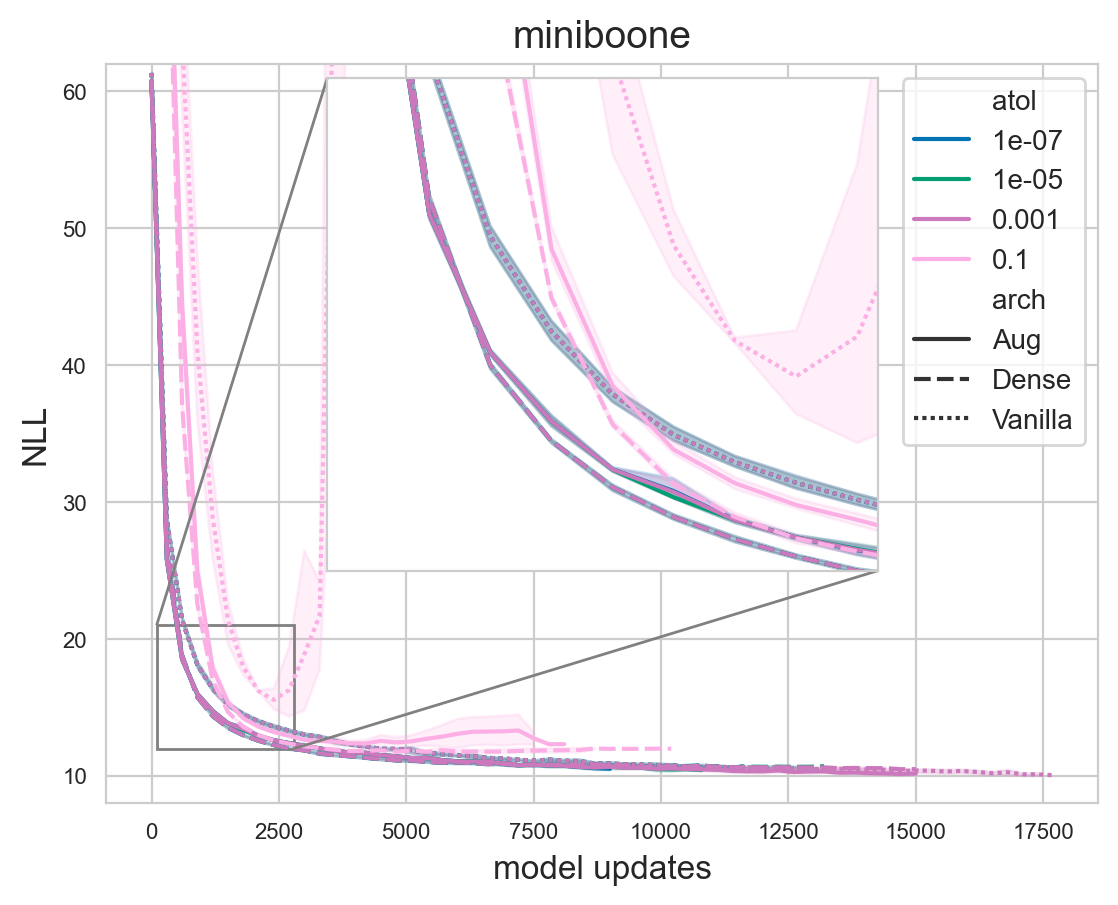}
    \includegraphics[width=0.49\textwidth]{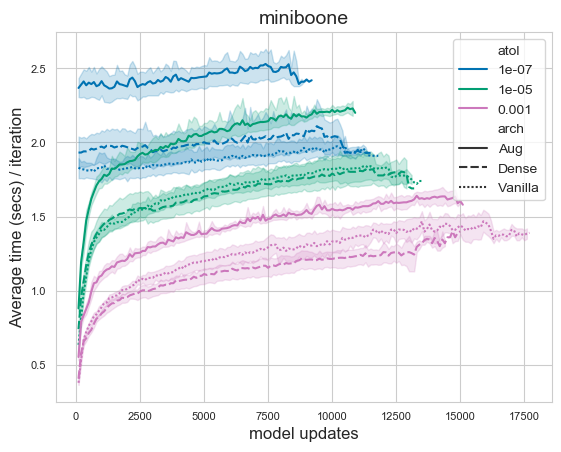}
    \includegraphics[width=0.49\textwidth]{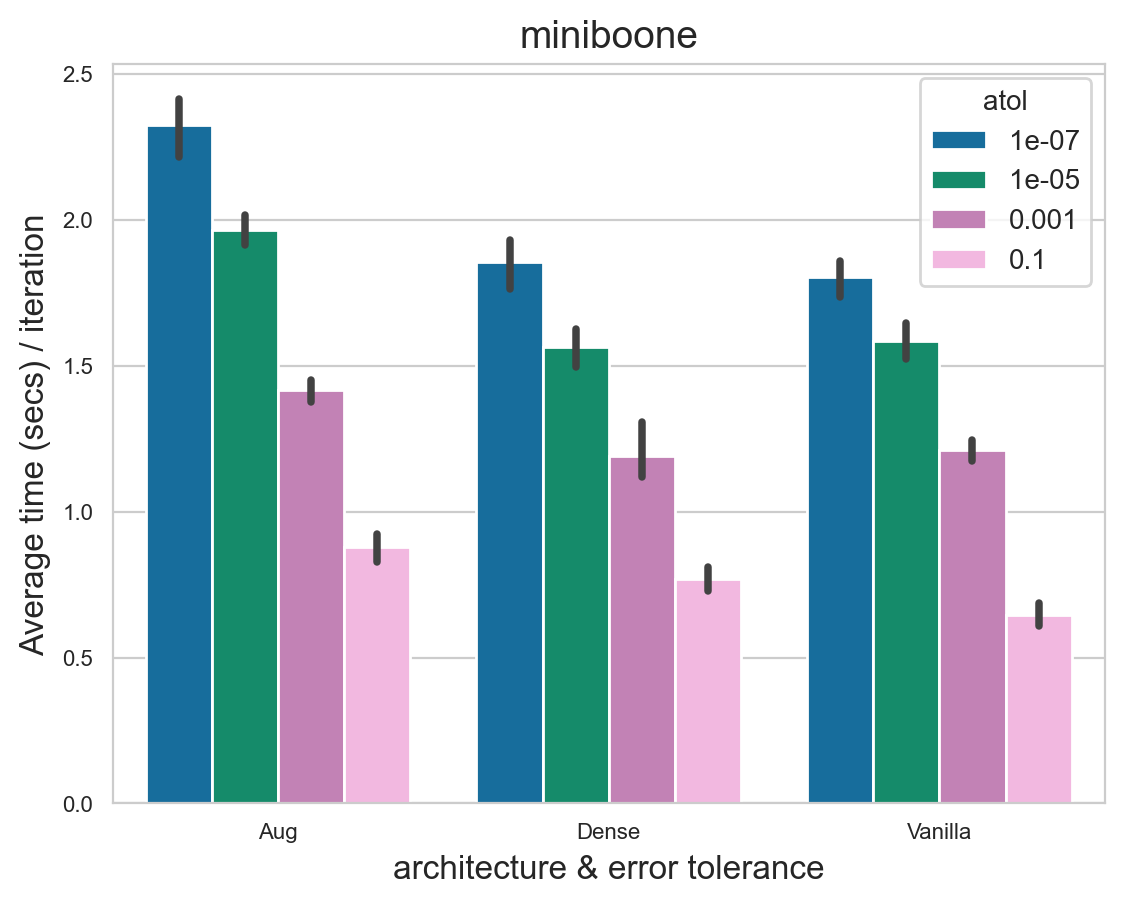}
    \caption{\small Ablation of different ICNN architectures and absolute error tolerance for conjugate gradient.
    Left: the average number of CG iterates (hvp calls) per flow layer (top row) and the corresponding average time (in seconds) per iteration (bottom row). 
    Top right: validation set negative log-likelihood (exact estimate).
    Notice that, for $\text{atol} = \num{1e-7}$, CG iterations cap at $43$ per flow layer; this is the dimensionality of the input data in the \textsc{Miniboone} dataset.
    Bottom right: per-iteration time (in seconds) averaged over all training steps.}
    \label{fig:ablate_atol_arch}
\end{figure}

\begin{figure}
    \centering
    \includegraphics[width=0.49\textwidth]{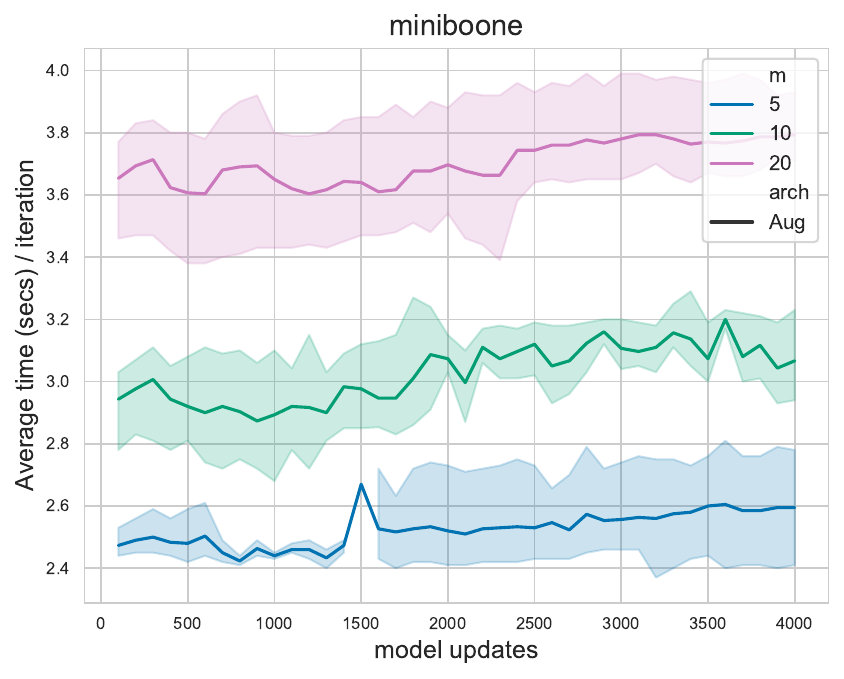}
    \caption{\small Ablation of the effect of number of eigenvalues used in SLQ on average time (in seconds) per iteration.}
    \label{fig:ablate_m_lanczos}
\end{figure}

\end{document}